\documentclass{article}

\usepackage{microtype}
\usepackage{graphicx}
\usepackage{subfigure}
\usepackage{booktabs} %

\usepackage{hyperref}

\usepackage[accepted]{icml2023}

\usepackage{amsmath}
\usepackage{amssymb}
\usepackage{mathtools}
\usepackage{amsthm}
\usepackage{bm}
\usepackage{bbm}
\usepackage{mathrsfs}
\usepackage{dsfont}
\usepackage{enumitem}
\usepackage{arydshln}
\usepackage[capitalize,noabbrev]{cleveref}

\newcommand{\bvb}{\boldsymbol{b}}
\newcommand{\ve}{\mathbf{e}}
\newcommand{\bve}{\boldsymbol{e}}
\newcommand{\vg}{\mathbf{g}}

\newcommand{\vh}{\mathbf{h}}

\newcommand{\bvv}{\boldsymbol{v}}
\newcommand{\vx}{\mathbf{x}}
\newcommand{\bvx}{\boldsymbol{x}}
\newcommand{\vy}{\mathbf{y}}
\newcommand{\bvy}{\boldsymbol{y}}
\newcommand{\vz}{\mathbf{z}}
\newcommand{\bvz}{\boldsymbol{z}}
\newcommand{\vepsilon}{\bm{\epsilon}}
\newcommand{\vxi}{\bm{\xi}}

\newcommand{\vphi}{\bm{\varphi}}

\newcommand{\vpsi}{\bm{\psi}}

\newcommand{\vzero}{\mathbf{0}}

\newcommand{\vA}{\mathbf{A}}
\newcommand{\vB}{\mathbf{B}}
\newcommand{\vI}{\mathbf{I}}
\newcommand{\vJ}{\mathbf{J}}
\newcommand{\vU}{\mathbf{U}}
\newcommand{\vV}{\mathbf{V}}
\newcommand{\vW}{\mathbf{W}}
\newcommand{\vP}{\mathbf{P}}
\newcommand{\vQ}{\mathbf{Q}}

\newcommand{\ko}{{(k)}}
\newcommand{\kp}{{(k+1)}}
\newcommand{\ro}{{(r)}}

\newcommand{\bbR}{\mathbb{R}}
\newcommand{\bbE}{\mathbb{E}}
\newcommand{\bbone}{\mathds{1}}
\newcommand{\avekT}{\frac{1}{T+1}\sum_{k=0}^T}
\newcommand{\sumin}{\sum_{i=1}^n}
\newcommand{\avein}{\frac{1}{n}\sum_{i=1}^n}

\newcommand{\Mod}[1]{\ (\mathrm{mod}\ #1)}
\newcommand{\Ours}{DSGD-CECA}

\theoremstyle{plain}
\newtheorem{theorem}{Theorem}[section]

\newtheorem{lemma}[theorem]{Lemma}

\theoremstyle{definition}
\newtheorem{definition}[theorem]{Definition}
\newtheorem{assumption}[theorem]{Assumption}
\newtheorem{remark}[theorem]{Remark}

\usepackage[textsize=tiny]{todonotes}

\allowdisplaybreaks

\icmltitlerunning{Communication-Optimal Decentralized SGD}

\begin{document}

\twocolumn[
\icmltitle{DSGD-CECA: Decentralized SGD with  Communication-Optimal Exact Consensus Algorithm}

\begin{icmlauthorlist}
\icmlauthor{Lisang Ding}{ucla}
\icmlauthor{Kexin Jin}{prin}
\icmlauthor{Bicheng Ying}{comp}
\icmlauthor{Kun Yuan}{pku,aisi,national}
\icmlauthor{Wotao Yin}{damo}
\end{icmlauthorlist}

\icmlaffiliation{ucla}{Department of Mathematics, University of California, Los Angeles, CA, USA}
\icmlaffiliation{prin}{Department of Mathematics, Princeton University, Princeton, NJ, USA}
\icmlaffiliation{comp}{Google Inc., Los Angeles, CA, USA}
\icmlaffiliation{pku}{Center for Machine Learning Research, Peking University, Beijing, P. R. China.}
\icmlaffiliation{aisi}{AI for Science Institute, Beijing, P. R. China}
\icmlaffiliation{national}{National Engineering Labratory for Big Data Analytics and Applications, Beijing, P. R. China}
\icmlaffiliation{damo}{Decision Intelligence Lab, Alibaba US, Bellevue, WA, USA}
\icmlcorrespondingauthor{Lisang Ding}{lsding@math.ucla.edu}
\icmlcorrespondingauthor{Wotao Yin}{wotao.yin@alibaba-inc.com}

\icmlkeywords{Machine Learning, ICML}

\vskip 0.3in
]

\printAffiliationsAndNotice{}  %

\begin{abstract}
Decentralized Stochastic Gradient Descent (SGD) is an emerging neural network training approach that enables multiple agents to train a model collaboratively and simultaneously. Rather than using a central parameter server to collect gradients from all the agents, each agent keeps a copy of the model parameters and communicates with a small number of other agents to exchange model updates. Their communication, governed by the communication topology and gossip weight matrices, facilitates the exchange of model updates. The state-of-the-art approach uses the dynamic one-peer exponential-2 topology, achieving faster training times and improved scalability than the ring, grid, torus, and hypercube topologies. However, this approach requires a power-of-2 number of agents, which is impractical at scale. In this paper, we remove this restriction and propose \underline{D}ecentralized \underline{SGD} with \underline{C}ommunication-optimal \underline{E}xact \underline{C}onsensus 
\underline{A}lgorithm (DSGD-CECA), which works for any number of agents while still achieving state-of-the-art properties. In particular, DSGD-CECA incurs a unit per-iteration communication overhead and an $\tilde{O}(n^3)$ transient iteration complexity. 
Our proof is based on newly discovered properties of gossip weight matrices and a novel approach to combine them with DSGD's convergence analysis. 
Numerical experiments show the efficiency of DSGD-CECA. 
\end{abstract}

\section{Introduction}
Decentralized computing \cite{tsitsiklis1986distributed,lopes2008diffusion,nedic2009distributed,dimakis2010gossip} is an essential subclass of distributed computing with no data fusion center.
{In scenarios where data and computational resources are distributed, decentralized computing enables each agent to process its local data and communicate with a selected group of other agents. This approach helps avoid the formation of central-agent-induced communication bottlenecks.} Without a central agent, however, the decentralized algorithm must achieve a global result through peer-to-peer interactions of the agents. Hence, the algorithm performance heavily depends on how effectively and efficiently the agents exchange their information.

\begin{table*}[t]
	    \centering 
		\caption{\small Comparison between DSGD over different commonly-used topologies. ``Static Exp.'': static exponential graph; ``O.-P. Exp.'': one-peer exponential graph; 
  ``DSGD-CECA-1P'': DSGD-CECA that supports 1-port communication model; ``DSGD-CECA-2P'': DSGD-CECA that supports 2-port communication model. 
  Undirected graphs can admit symmetric gossip matrices. If some graph has a  dynamic pattern, its associated communication matrix will vary at each iteration. Notation $\tilde{O}(\cdot)$ ignores all polylogarithmic factors.}
		\begin{tabular}{rccllc}
			\toprule
			\textbf{Topology} & \textbf{Connection} & \textbf{Pattern} & \textbf{Per-iter Comm.} & \hspace{-2mm}\textbf{Trans. Iters.} & \textbf{size $n$}\\ \midrule
			Ring \cite{nedic2018network}              & undirected          &                static                         &            $\Theta(1)$                &\hspace{-2mm}   $O(n^7)$                 &  arbitrary  \\ 
			Grid \cite{nedic2018network}              & undirected          &                static                         &            $\Theta(1)$                &\hspace{-2mm}   $\tilde{O}(n^5)$                 &   arbitrary \\ 
			Torus \cite{nedic2018network}              & undirected          &          static                               &         $\Theta(1)$                   &\hspace{-2mm}     ${O}(n^5)$              &   arbitrary   \\ 
			Hypercube \cite{trevisan2017lecture}              &undirected           &            static                             &       $\Theta(\ln(n))$                   &\hspace{-2mm}       $\tilde{O}(n^3)$         &      power of $2$     \\ 
			Static Exp.\cite{ying2021exponential} & directed &  static      & $\Theta(\ln(n))$         &\hspace{-2mm} $\tilde{O}(n^3)$   &  arbitrary\\
			O.-P. Exp.\cite{ying2021exponential} & directed &  dynamic (2-port)      & $1 $         &\hspace{-2mm} $\tilde{O}(n^3)$ & power of $2$ \\
			{\color{blue}DSGD-CECA-1P (Ours)} & {\color{blue}undirected} & {\color{blue}dynamic (1-port)}& {\color{blue}$1$}&\hspace{-2mm}    {\color{blue}$\tilde{O}(n^3)$}      & {\color{blue}even}\\
			{\color{blue}DSGD-CECA-2P (Ours)} & {\color{blue}directed} & {\color{blue}dynamic (2-port)} & {\color{blue}$1$}         &\hspace{-2mm} {\color{blue}$\tilde{O}(n^3)$}  & {\color{blue}arbitrary}\\
			\bottomrule
		\end{tabular}
		\label{Table:Summary}
	\end{table*}

{In scenarios where} the global goal is to compute an average across all agents,  {this challenge is identified as \textit{average consensus} or \textit{allreduce averaging}. Various optimal methods are established for a range of prevalent communication settings to address this problem.} 
The goal of this paper, however, is to accelerate \emph{decentralized SGD (DSGD)} \cite{chen2012diffusion,lian2017can,koloskova2020unified}, which is widely used in large-scale deep neural network training, by applying average consensus methods judiciously.

When a distributed SGD algorithm relies on a parameter server, the distributed agents have the same model parameters. 
However, when the scale of training requires us to use a large number of distributed agents, the parameter server becomes the bottleneck. Without a parameter server, the agents in decentralized SGD algorithms maintain the similarity among their copies of model parameters through message passing. The cost to make them the same among $n$ agents is at least $\lceil\log_2(n)\rceil$ rounds of message passing with each agent sending and receiving one message at each round, but this cost is unnecessary. 

Performing only one round of message passing after each mini-batch SGD step saves time though it causes the convergence of SGD to take more steps. It is {shown} in \cite{lian2017can,pu2019sharp,koloskova2020unified,ying2021exponential} that, for distributed smooth nonconvex objectives, {a decentralized approach} with one communication round per SGD step is slower than centralized SGD \emph{only} during an initial period of iterations, called the transient period. Afterward, SGDs with or without decentralization {tend to show} similar performance. {Given the expense and time-intensive nature of large-scale training, 
a practical DSGD should aim to minimize its transient period to enhance competitiveness. Recent SGD methods based on various communication topologies, each leading to different transient iterations, are proposed.}
We provide a summary in Table \ref{Table:Summary}.

Among different decentralized SGD algorithms, dynamic exponential-2 (also known as one-peer exponential-2) message passing \cite{assran2019stochastic,ying2021exponential} is currently state-of-the-art. For $n$ that is a power of 2, every agent sends messages to {\em one single} { designated} neighbor at each SGD iteration according to a subgraph taken from a cyclic sequence of $\log_2(n)$ base subgraphs. {This} dynamic exponential-2 message passing can reach \emph{exact} global averaging in $\log_2(n)$ rounds of communication. Furthermore, decentralized SGD based on dynamic exponential-2 graph can obtain the state-of-the-art balance between per-iteration communication and transient iteration complexity \cite{,ying2021exponential}; it only incurs a unit communication overhead per iteration and $O(n^3\log^4_2(n))$ transient iterations, both of which are nearly the best among DSGDs {implemented with} other commonly-used topologies. 

Unfortunately, some excellent results of dynamic exponential-2 message passing \emph{no longer} hold when $n$ is not a power of 2, e.g., $n=10$ or $100$, including finite-time convergence for average consensus and DSGD convergence guarantees. As $n$ increases, the {number of instances where $n$ is a power of 2 becomes} increasingly scarce. {For example, if} one has a sizeable deep-learning training task that fits well into 20 GPU nodes at hand, the current choice is to either {utilize} a less-efficient DSGD algorithm or {scale} up to 32 nodes to run the most efficient DSGD algorithm. 

{Furthermore, dynamic exponential-2 message passing operates exclusively within the 2-port communication model over \emph{directed} topologies. In this model, each agent sends information to an agent and receives information from another agent simultaneously in each round.} {Contrarily, it is not applicable} to 1-port model over \emph{undirected} topologies where each agent sends and receives information to/from the {\em same} agent during each communication round. 1-port model is typically more efficient in full-duplex communication systems, and it admits symmetric gossip communication matrices, which are required by many important decentralized optimization algorithms such as decentralized ADMM \cite{shi2014linear}, EXTRA \cite{shi2015extra}, and Exact-Diffusion/D$^2$ \cite{yuan2017exact1,li2017decentralized,tang2018d}.

Since dynamic exponential-2 suffers from the above limitations, we ask the following question. {\em Can we develop new DSGD algorithms that work for any $n$ (or at least any even $n$), support both 1-port and 2-port communication models, and inherit the nice properties of dynamic exponential-2?} This paper provides affirmative answers.

\subsection{Contributions}

This paper introduces a novel DSGD algorithm that works for any $n$ (or any even $n$ under the 1-port communication model) and achieves state-of-the-art balance between per-iteration communication and transient iteration complexity. Our main contributions are listed as follows.

\begin{itemize}[leftmargin=10pt]
    \item We revisit a less well-known but \underline{c}ommunication-optimal \underline{e}xact \underline{c}onsensus \underline{a}lgorithm (CECA) proposed in \cite{bar1993optimal}. CECA requires $\lceil\log_2(n)\rceil$ rounds of message passing (which is optimal and cannot be further reduced) to achieve global averaging. {The original CECA is restricted to 2-port communication. We improve this algorithm to 1-port communication for any even $n$ and show that it achieves exact average consensus in $\lceil\log_2(n)\rceil$ rounds of message passing.}
    
    \item We next judiciously apply CECA into decentralized learning and propose DSGD-CECA. 
    To save communications, our algorithm only conducts one single round of CECA message passing after each mini-batch SGD step. 
    To guarantee convergence, our algorithm introduces a new strategy that maintains copies of local models, thereby inheriting the periodic global averaging property from CECA. Besides,
    DSGD-CECA works for any $n$ under the 2-port communication model and any even $n$ under the 1-port model. {Importantly, our DSGD-CECA supports both directed graphs and undirected graphs.}

    \item We further establish that DSGD-CECA \ incurs a $\Theta(1)$ per-iteration communication overhead and $\tilde{O}(n^3)$ transient iteration complexity; both of which are optimal compared to the baselines; see Table \ref{Table:Summary}. The convergence analysis of DSGD-CECA\ is non-trivial because the gossip weight matrix of CECA is not doubly-stochastic. 
    {However, our analysis leverages newly discovered properties of this matrix, which helps resolve analysis challenges significantly.}
\end{itemize}

\textbf{Notes.} This paper considers
deep neural network training within high-performance data-center clusters, in which the network topology can be fully controlled and any two GPUs can communication (through network switches) when necessary. The proposed algorithms may not work well in scenarios (e.g., wireless sensor networks, internet of vehicles, etc.)  where connection constraints exist. In addition, this paper studies deterministic message passing listed in Table \ref{Table:Summary}. There are recent works that study DSGD with stochastic message passing (reviewed below). However, stochastic message passing is less easy to control and implement.
Moreover, it can cause DSGD to be arbitrarily slow with non-zero probability.

\section{Preliminary and Related Work}

\subsection{Preliminary}
\textbf{Problem. }Consider $n$ computing agents working collaboratively to solve the distributed optimization problem.
\begin{equation}
\label{eq:opt-problem-setting}
\operatorname{min}_{\bvx\in\bbR^d} f(\bvx)=\avein f_i(\bvx),
\end{equation}
where $f_i(\bvx)=\bbE_{\xi_i\sim \mathcal{D}_i}F(\bvx;\xi_i)$.
In the above problem, $\xi_i$ denotes random local data kept at agent $i$, and it is sampled from distribution $\mathcal{D}_i$. It is common that $\mathcal{D}_i \neq \mathcal{D}_j$ when $i \neq j$, which causes the data heterogeneity issue in distributed learning problems.
\begin{figure}[t]
\begin{center}
\centerline{\includegraphics[width=0.85\columnwidth]{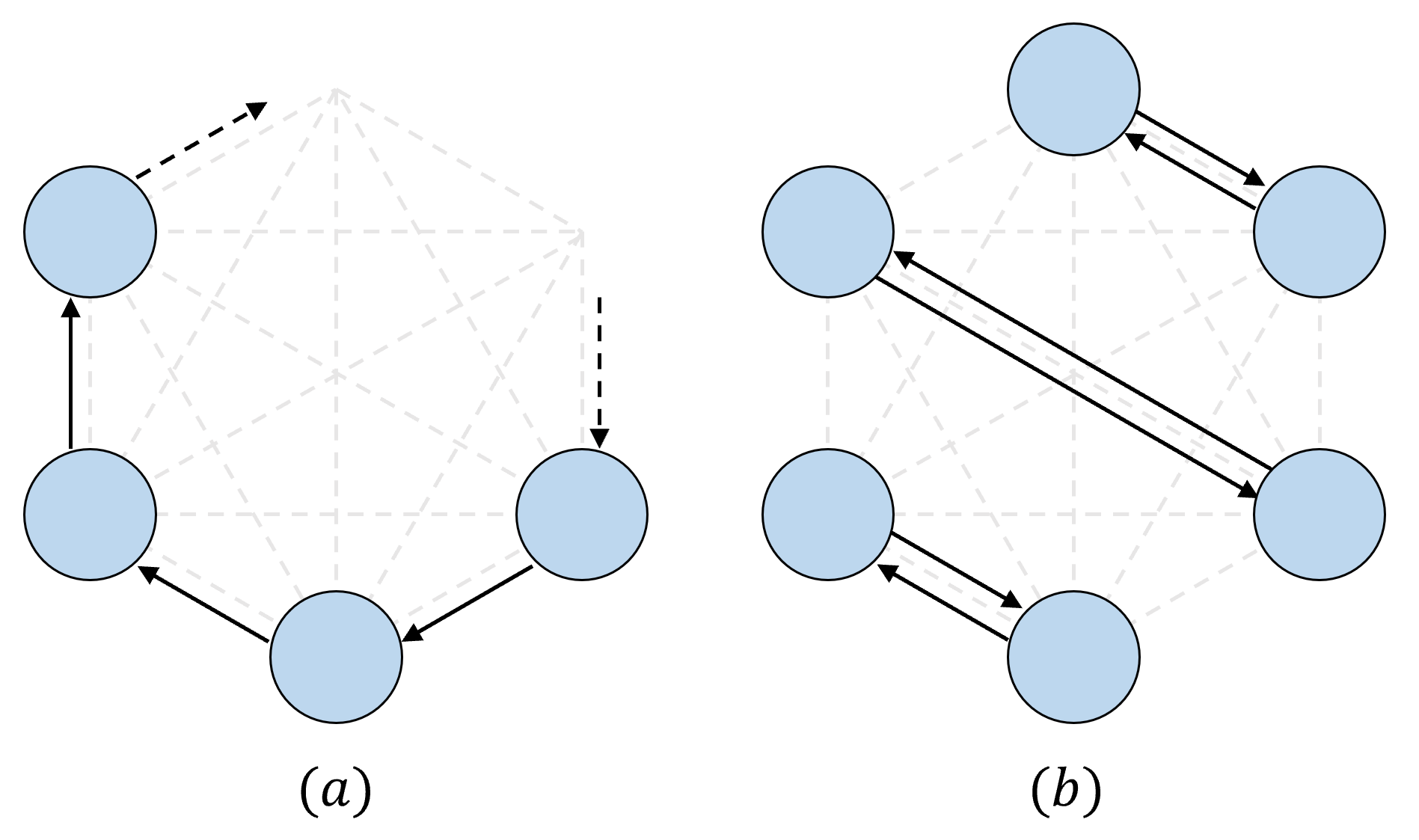}}
\vskip -0.2in
\caption{
(a) The 2-port communication model. Each agent sends information to one neighbor and receives information from another different neighbor per communication round. (b) The 1-port communication model. Each agent is paired with one single neighbor and exchanges information with it. }
\label{fig:1p-example-n=6-0}
\end{center}
\vskip -0.3in
\end{figure}

\textbf{Network topology and communication matrix.} Decentralized optimization depends on partial averaging {among} connected agents, {the relationships of which are dictated by the network topology—either directed or undirected—that links all the agents.} We let $\vP \in \mathbb{R}^{n\times n}$ denote a communication matrix that characterizes the sparsity and connectivity of the network topology. To this end, we let $\vP_{i,j} = 1$ if agent $j$ can send information to agent $i$ otherwise $\vP_{i,j} = 0$. 

\textbf{Communication models.} This paper will develop decentralized SGD algorithms based on the following communication models. 
\begin{itemize}[leftmargin=10pt]
    \item \textbf{1-port model}. This model {applies to} undirected network topologies. {In this model, during each communication, each agent communicates bidirectionally, both sending and receiving information to and from the same agent, as shown in } Fig.~\ref{fig:1p-example-n=6-0}(b). 1-port model admits symmetric communication matrix which are required in many popular decentralized optimization methods, and it is typically more efficient than 2-port model in full-duplex communication systems. OU-EquiDyn \cite{song2022communication} {adheres to} the 1-port communication model.

    \item \textbf{2-port model}. This model {operates} over directed network topologies. Each agent in this model sends information to an agent and receives information from another agent simultaneously in each round, {as illustrated in} Fig.~\ref{fig:1p-example-n=6-0}(a). Both dynamic exponential-2 \cite{ying2021exponential} and OD-EquiDyn \cite{song2022communication} follow the 2-port  model. 
\end{itemize}

It is worth noting that both 1-port and 2-port models are efficient in communication. They only incur $\Theta(1)$ communication overhead per iteration since each agent in these models only talks with one single neighbor.

\begin{figure}[t]
\begin{center}
\centerline{\includegraphics[width=0.7\columnwidth]{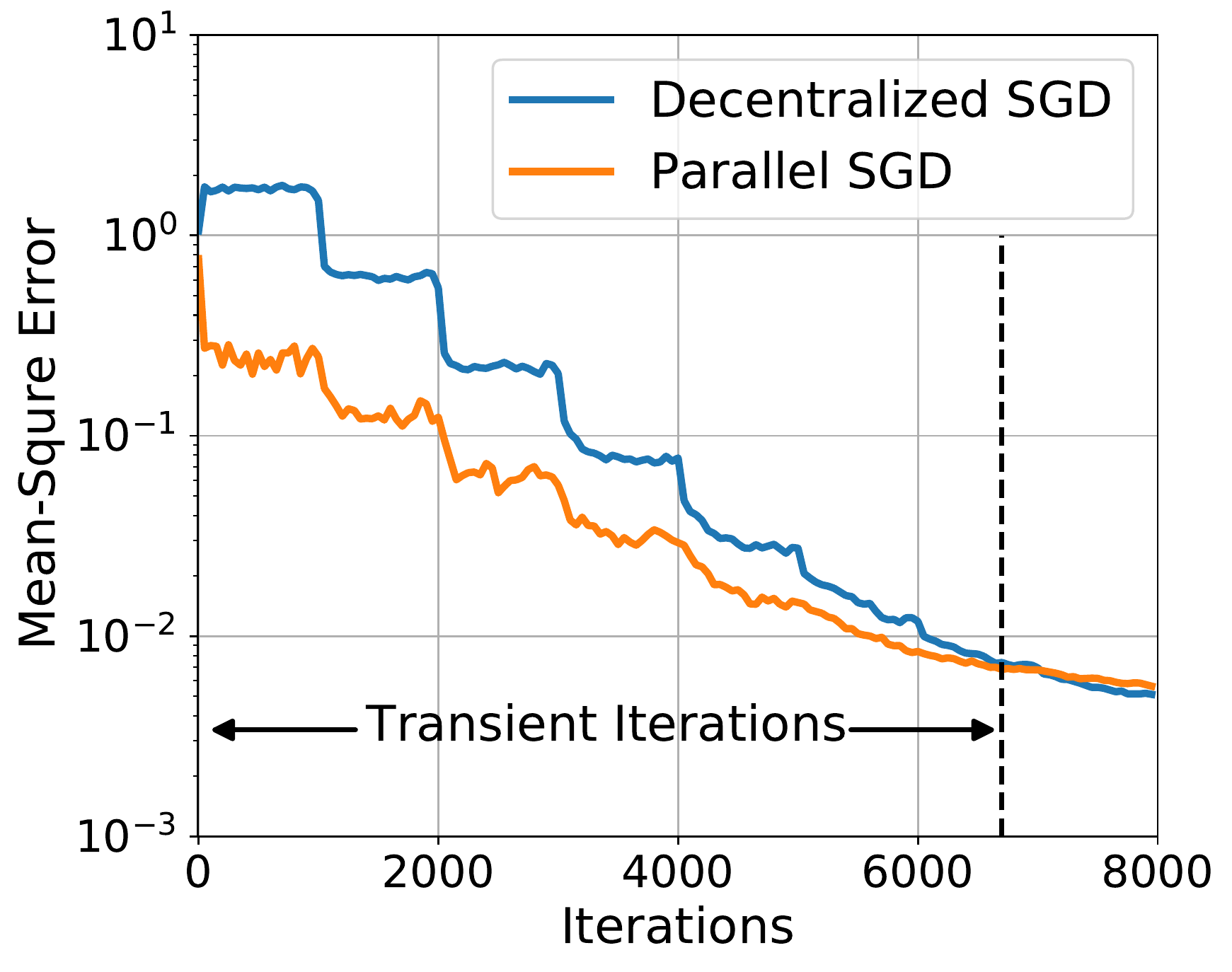}}
\vskip -0.2in
\caption{
Illustration of the transient iterations \cite{ying2021exponential}.}
\label{fig:transient}
\end{center}
\vskip -0.3in
\end{figure}

\textbf{Decenralized SGD.} Decentralized SGD (DSGD), an emerging training technique for large-scale deep learning, relaxes the global averaging step in traditional parallel SGD to inexact partial averaging within neighborhood. It is characterized by its substantially less (and thus faster) communication every iteration. The less neighbors each agent needs to talk with (i.e., the sparser the network topology is), the faster the per-iteration communication is. However, the communication efficiency in DSGD comes at a cost -- slower convergence since partial averaging is less effective to aggregate information. 
It is found in \cite{lian2017can,pu2019sharp,koloskova2020unified} that DSGD can achieve the same convergence rate as parallel SGD after some {\em transient iterations}, see Fig.~\ref{fig:transient} for an illustration. The longer the transient period is, the slower the algorithm converges. This paper targets to develop new algorithms that attain minimal transient iteration complexity with little  communication overhead per iteration. 

\textbf{Assumptions. }We introduce several standard assumptions for problem \eqref{eq:opt-problem-setting}.
\begin{assumption}[\sc Lipschitz smoothness]
\label{ass:Lip}
Each local function $f_i$ is $L-$smooth, i.e., $\|\nabla f_i(\bvx)-\nabla f_i(\bvy)\|\leq L\|\bvx-\bvy\|$ for any $\bvx,\bvy\in\bbR^d$.
\end{assumption}
\begin{assumption}[\sc Gradient noise]
\label{ass:grad-noise}
Random data variable $\xi^\ko_i$ is independent of  each other for any $k$ and $i$. The gradient noise satisfies
$\bbE_{\xi_i\sim \mathcal{D}_i}[\nabla F(\bvx;\xi_i)] = \nabla f_i(\bvx)$, $\bbE_{\xi_i\sim \mathcal{D}_i}\|\nabla F(\bvx;\xi_i)-\nabla f_i(\bvx)\|^2\leq \sigma^2$, for any $\bvx\in \bbR^d$.
\end{assumption}
\begin{assumption}[\sc Data heterogeneity]
\label{ass:data-hetero}
The local functions satisfies $\avein\|\nabla f_i(\bvx)-\nabla f(\bvx)\|^2\leq b^2$ for any $\bvx\in\bbR^d$ and $i$. 
\end{assumption}

\textbf{Notations.} Throughout the paper, we let $[n] = \{1, \cdots, n\}$ and define a $\mathrm{mod}$ operation that returns a value in $[n]$ as 
\begin{align*}
i\ \mathrm{mod}\ n = 
\left\{
\begin{array}{ll}
\ell & \mbox{if $i = kn + \ell$ for some $\ell \in [n-1]$}, \\
n & \mbox{if $i = kn $}.
\end{array}
\right.
\end{align*}
where $k$ is an integer. When $i$ is the agent index, we will simplify $(i-\ell)\ \mathrm{mod}\ n $ as $i-\ell$ for any $\ell \in [n]$. For example, suppose $n=6$ and $i=1$, it holds that $i-1 = 6$ and $i-2 = 5$.

\subsection{Related work}

\textbf{Decentralized deep training.} 
Decentralized SGD algorithms \cite{lopes2008diffusion,yuan2016convergence,lian2017can,koloskova2019decentralized} are widely used to accelerate large-scale deep training. {These algorithms have been} extended to various practical settings,  {including those with} directed  \cite{assran2019stochastic} and time-varying \cite{kong2021consensus,ying2021exponential,koloskova2020unified} network topologies, asynchronous model updating \cite{lian2018asynchronous,niwa2021asynchronous}, and momentum acceleration \cite{lin2021quasi,yuan2021decentlam}. 
{However,} DSGD suffers from data heterogeneity issues \cite{koloskova2020unified,yuan2020influence} {in the meanwhile.} Various advanced techniques such as EXTRA \cite{shi2015extra}, Exact-Diffusion/D$^2$ \cite{yuan2017exact1,li2017decentralized,yuan2021removing,tang2018d}, and gradient-tracking \cite{di2016next,xu2015augmented,nedic2017achieving,qu2018harnessing,xin2020improved,alghunaim2021unified} are proposed to {mitigate the impact of} data heterogeneity and {thereby accelerating} the DSGD convergence. 

\textbf{Message passing with asymptotic consensus.} Decentralized learning methods are typically based on gossip averaging. {While gossip averaging allows for quick per-iteration communication when running over topologies such as rings, grids, and torus, the rate of convergence towards the average consensus slows as $n$ increases \cite{nedic2018network}.} The hypercube graph \cite{trevisan2017lecture} maintains a nice balance between communication efficiency and consensus rate. But a hypercube cannot be formed when network size $n$ is not a power of $2$. The static exponential graph \cite{ying2021exponential}, 
{on the other hand, works for any $n$, but only admits directed weight matrices.}
Stochastic message passing is also widely used in decentralized learning. The Erdos-Renyi graph \cite{nachmias2008critical,benjamini2014mixing,nedic2018network} and the geometric random graph \cite{beveridge2016best,boyd2005mixing} are two representatives. A recent work \cite{song2022communication} proposes a state-of-the-art family of EquiTopo graphs that incur $\Theta(1)$ communication overhead per iteration and enjoy a network-size independent consensus rate. However, these stochastic message passing protocols can be difficult to control. {Moreover,} some realizations of these random protocols can be arbitrarily slow to achieve asymptotic consensus with non-zero probabilities.

\textbf{Message passing with exact consensus. }
{The concept of} exact consensus (also known as allreduce averaging) is extensively studied within the high-performance computing community. {This approach} can achieve exact global averaging with {a finite number of }communication rounds. Well-known methods include tree-allreduce \cite{ben2019demystifying}, ring-allreduce \cite{patarasuk2009bandwidth} and  BytePS \cite{jiang2020unified}. Recent works start integrating exact consensus techniques to decentralized optimization to boost performance. For example, \cite{ying2021exponential} utilizes  dynamic exponential-2 to balance communication efficiency and aggregation effectiveness in DSGD. Generally speaking, it is non-trivial to develop new decentralized algorithms with exact consensus techniques, mainly because they do not contribute doubly stochastic weight matrices.

\section{Communication-Optimal Exact Consensus}
\subsection{2-port optimal exact consensus}
\label{sec:CECA-allreduce-alg}
This section revisits the optimal message passing algorithm CECA \cite{bar1993optimal} for a 2-port communication system with $n$ agents, where $n$ can be any positive integer.

\textbf{Problem statement.} Letting each agent $i$ hold a local variable $u_i$, our target is to let each agent obtain $\bar{u} = \frac{1}{n}\sumin u_i$ after $\tau=\lceil \log_2 n \rceil$ rounds of communication.

\textbf{Auxiliary variables.} {We construct several auxiliary variables utilized in CECA.}
\begin{itemize}[leftmargin=10pt]
\item We convert $n-1$ to a binary number as 
\begin{equation}
\label{eq:bin-n-1}
n-1 = (\delta_0\,\delta_1\,\cdots\, \delta_{\tau-1} )_2
\end{equation}
where $\delta_0 \neq 0$ is the most significant bit and $\delta_{\tau-1}$ is the least significant bit,  e.g., $n-1=(101)_2$ when $n=6$.
\item We set $n_0 = 0$ and calculate $\{n_{r+1}\}_{r=0}^{\tau-1}$ by 
\begin{equation}
\label{eq:iter-nr}
n_{r+1}=2n_r+\delta_r,\quad r=0,1,2,\ldots,\tau-1.
\end{equation}
It is easy to verify that $n_\tau = n-1$. For example, if $n=6$, then $\tau=\lceil \log_2 6 \rceil=3$ and $n_1=1,n_2=2,n_3=5$.
\item Let each agent $i$ maintain variables $I_i^{(r)}$ and $J^\ro_i$ at iteration $r$, and initialize them as $I_i^{(0)} = u_i$ and $J^{(0)}_i = 0$. 
\end{itemize}

\textbf{Main idea.} For any $r=0,1,\cdots, \tau-1$, CECA will always guarantee that 
\begin{equation}
\label{z7u23bn}
I_i^{(r)} = \frac{1}{n_r+1}\sum_{j=0}^{n_r}u_{i-j}, \quad 
J_i^{(r)} = 
\left\{
\begin{split}
    &\frac{1}{n_r}\sum_{j=1}^{n_r}u_{i-j},\,r\geq 1\\
    &0,\,r=0
\end{split}
\right..
\end{equation}
It is observed that $I_i^{(r)}$ always keeps the average of agent $i$ and its $n_r$ previous neighbors, while $J_i^{(r)}$ keeps the average of agent $i$'s $n_r$ previous neighbors (but not including $u_i$).  When $r = \tau-1$, it holds that $I_i^{(r)} = \frac{1}{n}\sumin u_i$ and hence each agent will reach the average consensus.

\textbf{Main recursions.} To guarantee \eqref{z7u23bn}, CECA will conduct the following recursions for each $r=0,1,\cdots, \tau-1$.
\begin{align*}
&\mbox{If $\delta_r=1$ update }\left\{
\begin{array}{lc}
I^{(r+1)}_i \hspace{0.5mm} =\frac{1}{2}I^\ro_i+\frac{1}{2}I^\ro_{i-n_r-1}     &  \vspace{1mm}\\
J^{(r+1)}_i=\frac{n_r}{2n_r+1}J^\ro_i \hspace{-0.5mm}+\hspace{-0.5mm} \frac{n_r+1}{2n_r+1}I^\ro_{i-n_r-1}     & 
\end{array}
\right. \vspace{2mm}\\
&\mbox{If $\delta_r=0$ update }\left\{
\begin{array}{lc}
I^{(r+1)}_i=\frac{n_r+1}{2n_r+1}I^\ro_i+\frac{n_r}{2n_r+1}J^\ro_{i-n_r}     &  \vspace{1mm}\\
J^{(r+1)}_i=\frac{1}{2}J^\ro_i+\frac{1}{2}J^\ro_{i-n_r}     & 
\end{array}
\right.
\end{align*}
More details on CECA as well as illustrating examples can be referred to Appendix \ref{append:CECA-2-port}. 

\textbf{Communication patterns.} From the main CECA recursion listed above, it is observed that each agent will follow a 2-port communication model. To better capture the communication pattern, we let $\vQ^\ro$ denote the communication matrix employed at CECA round $r$. If agent $j$ sends information to agent $i$, {we set} $\vQ^\ro_{i,j}=1$; otherwise, $\vQ^\ro_{i,j}=0$.
\begin{equation}
\label{eq:2p-gossip-matrix-d1}
\mbox{If $\delta_r\hspace{-0.7mm}=\hspace{-0.7mm}1$ then }\vQ^\ro_{i,j}\hspace{-1mm}=\hspace{-1mm}
\left\{
\begin{array}{cc}
    \hspace{-2mm}1, & \textup{if $i\hspace{-0.6mm}-\hspace{-0.6mm}j\equiv n_r\hspace{-0.6mm}+\hspace{-0.6mm}1 \Mod{n}$} \\
    \hspace{-2mm}0, & \textup{otherwise}
\end{array}
\right.
\end{equation}
\begin{equation}
\label{eq:2p-gossip-matrix-d0}
\mbox{If $\delta_r\hspace{-0.7mm}=\hspace{-0.7mm}0$ then }\vQ^\ro_{i,j}\hspace{-0.7mm}=\hspace{-0.7mm}
\left\{
\begin{array}{cc}
    1, & \textup{if $i-j\equiv n_r \Mod{n}$} \\
    0, & \textup{otherwise}
\end{array}
\right.
\end{equation}
The matrix $\vQ^\ro$ is a permutation matrix that reflects the dynamic topology for message exchanging, as illustrated in Fig.~\ref{fig:2p-example-n=6-2} for the case when $n=6$. Additionally, the matrix $\vQ^\ro$ will facilitate the DSGD-CECA development.

\begin{figure}[t]
\vskip 0.2in
\begin{center}
\centerline{\includegraphics[page=1, width=\columnwidth]{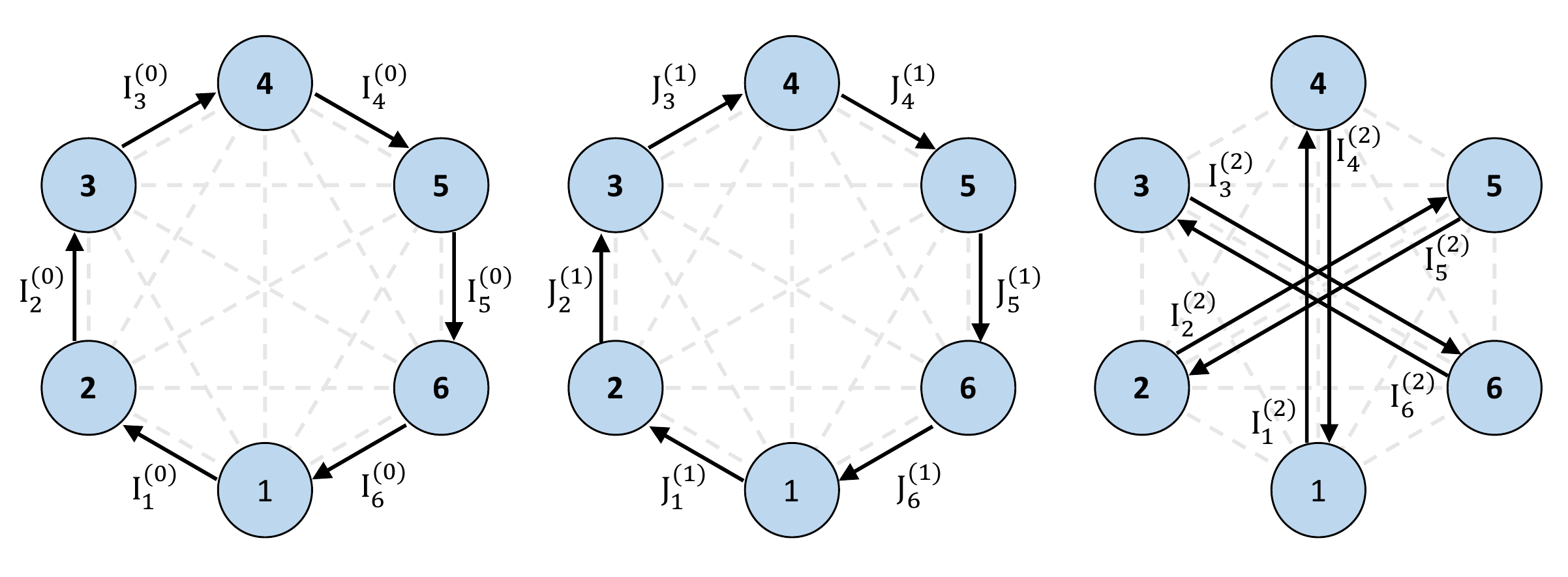}}
\vspace{-0.2mm}
\caption{
An example of CECA for $n=6$ agents. 
The method conducts 3 communication rounds to reach average consensus. 
The arrows and their labels indicate  the information flows.}
\label{fig:2p-example-n=6-2}
\end{center}
\vskip -0.4in
\end{figure}

\subsection{1-port optimal exact consensus}
\label{sec:1p-opt-allreduce-ave}
{The vanilla CECA introduced in \cite{bar1993optimal}, referred to here as CECA-2P, exclusively supports the 2-port communication model.} {In this section, we} will develop a new variant, CECA-1P, that supports 1-port communication model over undirected topology. CECA-1P enables each agent to reach average consensus after $\tau=\lceil \log_2 n \rceil$ rounds of communication when $n$ is even.

\textbf{Main recursions.} To achieve average consensus, CECA-1P introduces the same auxiliary variables as CECA-2P. The main idea of CECA-1P is similar to CECA-2P. In round $r$, agent $i$ pairs with agent $i+2n_r + 1$ if $i$ is odd, otherwise, it pairs with agent $i-2n_r - 1$. Following this, each pair of agents exchanges information with each other. If $\delta_r = 1$, they exchange $I^{(r)}$; otherwise they exchange $J^{(r)}$ instead. {We let ${\omega^\ro_i}$ denote the agent sending a message to agent $i$ in the $r^{\textup{th}}$ round.} CECA-1P conducts the following recursions for each $r=0,\cdots, \tau-1$.
\begin{align*}
&\mbox{If $\delta_r=1$ update }\left\{
\begin{array}{lc}
I^{(r+1)}_i=\frac{1}{2}I^\ro_i+\frac{1}{2}I^\ro_{\omega^\ro_i}     &  \vspace{1mm}\\
J^{(r+1)}_i=\frac{n_r}{2n_r+1}J^\ro_i+\frac{n_r+1}{2n_r+1}I^\ro_{\omega^\ro_i}     & 
\end{array}
\right. \vspace{2mm}\\
&\mbox{If $\delta_r=0$ update }\left\{
\begin{array}{lc}
I^{(r+1)}_i=\frac{n_r+1}{2n_r+1}I^\ro_i+\frac{n_r}{2n_r+1}J^\ro_{\omega^\ro_i}    &  \vspace{1mm}\\
J^{(r+1)}_i=\frac{1}{2}J^\ro_i+\frac{1}{2}J^\ro_{\omega^\ro_i}    & 
\end{array}
\right.
\end{align*}

More details on CECA as well as illustrating examples can be referred to Appendix \ref{append:CECA-2P}. 

\textbf{Communication patterns.} The communication matrix $\vQ^\ro$ in CECA-1P is given by
\begin{equation}
\label{eq:1p-gossip-matrix}
\vQ^\ro_{i,j}=
\left\{
\begin{array}{cc}
    1, & i \textup{ odd},~ j=i+2n_r+1 \\
    1, & i \textup{ even},~ j=i-2n_r-1 \\
    0, & \textup{otherwise}
\end{array}
\right.
\end{equation}
Note that $\vQ^\ro$ is a symmetric matrix. Fig.~\ref{fig:1p-example-n=6-1P} illustrates the communication pattern for CECA-1P when $n=6$.

\begin{figure}[t]
\vskip 0.2in
\begin{center}
\centerline{\includegraphics[page=2, width=\columnwidth]{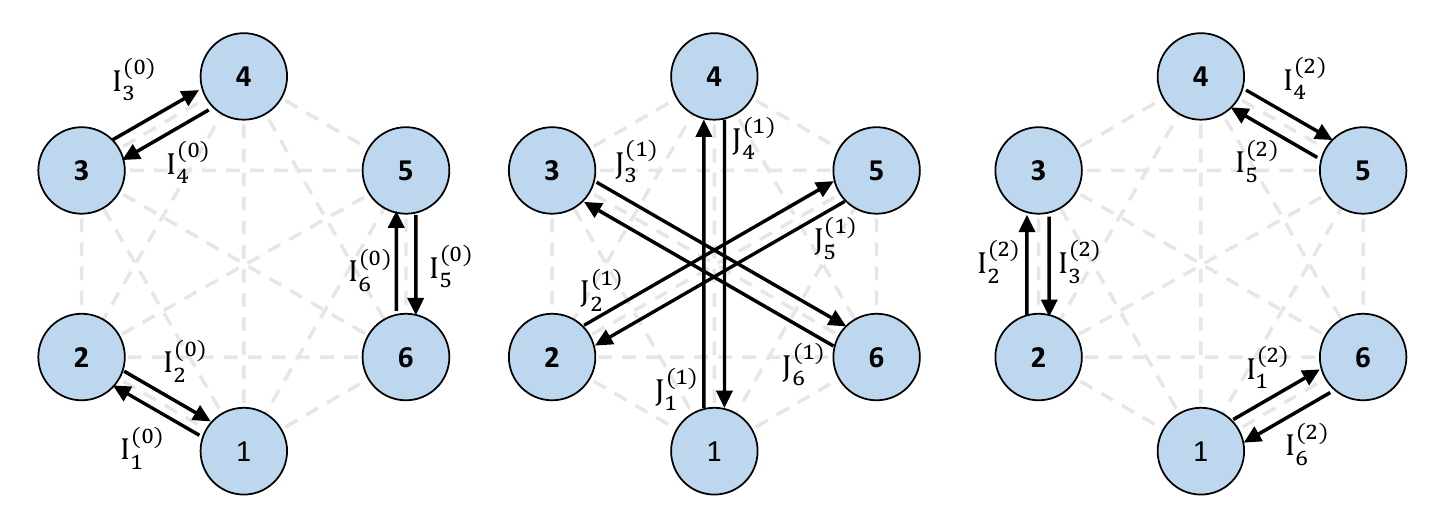}}
\caption{
An example of CECA-1P for $n=6$ agents. It is observed that the resulting topology is undirected per iteration. %
}
\label{fig:1p-example-n=6-1P}
\end{center}
\vskip -0.2in
\end{figure}

\section{\Ours\ Algorithm}
This section develops a novel DSGD algorithm based on CECA, as discussed in \S \ref{sec:CECA-allreduce-alg}. The resulting CECA-DSGD works with any $n$ in the 2-port communication model and any even $n$ in the 1-port model. In either scenario, \Ours\ incurs $\Theta(1)$ per-iteration communication overhead and $\tilde{O}(n^3)$ transient iteration complexity; both of which are nearly the best compared to  baselines, see Table \ref{Table:Summary}. 

\textbf{Challenges.} It is highly non-trivial to integrate CECA to DSGD due to the following  challenges. First, \Ours\ splits CECA to a sequence of separate communication rounds, and it only performs one single CECA message passing after each mini-batch SGD steps. It is unknown whether this strategy will deteriorate the optimal communication efficiency of CECA. Second, CECA requires to maintain two sets of variables, i.e., $I^{(r)}$ and $J^{(r)}$, to enforce average consensus. It is unknown what auxiliary variables shall be introduced to DSGD to facilitate the integration with CECA. Third, the introduction of CECA to DSGD will crash the doubly-stochastic property of the gossip weight matrix in DSGD. For this reason, traditional theories in  \cite{koloskova2020unified,alghunaim2021unified,ying2021exponential} cannot be utilized to analyze \Ours. This section will resolve all these challenges. 

\subsection{Algorithm development}

\textbf{Algorithm description.}
We first introduce \Ours\ that supports 2-port communication system. To this end, we let each agent $i$ maintain a local model $\bvx_i$ and an additional auxiliary model $\bvy_i$ to facilitate the integration with CECA. To better present the algorithm, we introduce the following notations
\begin{equation*}
\resizebox{ \columnwidth}{!} {
$\begin{split}
\vx^\ko&=[(\bvx^\ko_1)^\top;(\bvx^\ko_2)^\top;\cdots;(\bvx^\ko_n)^\top], \\
\vy^\ko&=[(\bvy^\ko_1)^\top;(\bvy^\ko_2)^\top;\cdots;(\bvy^\ko_n)^\top], \\
\nabla F(\vx^\ko; \vxi^\ko)&= [(\nabla f(\bvx^\ko_1;\xi_1^\ko)^\top;\cdots;(\nabla f(\bvx^\ko_n;\xi_n^\ko)^\top],\\
\nabla F(\vy^\ko; \vxi^\ko)&= [(\nabla f(\bvy^\ko_1;\xi_1^\ko)^\top;\cdots;(\nabla f(\bvy^\ko_n;\xi_n^\ko)^\top].
\end{split}$
}
\end{equation*}

With these notations, \Ours\ is listed in Algorithm \ref{alg:dsgd-ceca}. 
\begin{algorithm}[tb]
   \caption{DSGD-CECA}
   \label{alg:dsgd-ceca}
\begin{algorithmic}
   \STATE {\bfseries Initialize:} Randomly initialize $\bvx_i^{(0)}$; set $\bvy_i^{(0)}=\bvx_i^{(0)}$; {set learning rate $\gamma$;} compute $\tau=\lceil \log_2n \rceil$; convert $n-1 = (\delta_0\,\cdots\, \delta_{\tau-1} )_2$;
   \FOR{$k=0,1,2,\ldots,T$}
   \STATE $r=k\Mod{\tau}$;
   \IF{$\delta_r=1$}
   \STATE Compute $\vg^\ko=\nabla F(\vx^\ko;\vxi^\ko)$;
   \STATE Let $\vz^\ko=\vx^\ko$ and  $\ve^\ko=\vg^\ko$;
   \STATE Sample communication matrix $\vP^\ko$ as $\vQ^\ro$ in \eqref{eq:2p-gossip-matrix-d1};
   \STATE Let $a^\ko=\frac{1}{2}$ and $b^\ko=\frac{n_r}{2n_r+1}$;
   \ELSIF{$\delta_r=0$}
   \STATE Compute $\vh^\ko=\nabla F(\vy^\ko;\vxi^\ko)$;
   \STATE Let $\vz^\ko=\vy^\ko$ and $\ve^\ko=\vh^\ko$;
   \STATE Sample communication matrix $\vP^\ko$ as $\vQ^\ro$ in \eqref{eq:2p-gossip-matrix-d0};
   \STATE Let $a^\ko=\frac{n_r+1}{2n_r+1}$ and  $b^\ko=\frac{1}{2}$;
   \ENDIF
   \STATE Communicate $\vz^\ko, \ve^\ko$ among agents following $\vP^\ko$;
   \STATE Update $\vx^\kp$ as in \eqref{eq:x-iter-dsgd}; 
   \STATE 
       Update $\vy^\kp$ as in \eqref{eq:y-iter-dsgd};
   \ENDFOR
\end{algorithmic}
\end{algorithm}

\Ours\ follows a similar communication protocol as CECA. The communicated information varies with iterations, and it is determined by $\delta_r$. Agent $i$ sends information $(\bvz^\ko_i-\gamma\bve^\ko_i)$ to agent $j$ when  $\vP^\ko_{i,j}=1$ and receives information $(\bvz^\ko_l-\gamma\bve^\ko_l)$ from agent $l$ where $\vP^\ko_{l,i}=1$. After the information communication, $\vx$ and $\vy$ will update as follows
\begin{align}
\label{eq:x-iter-dsgd}
&\begin{aligned}
\vx^\kp&=a^\ko(\vx^\ko-\gamma \ve^\ko)+\\
&\quad(1-a^\ko)\vP^\ko(\vz^\ko-\gamma\ve^\ko),
\end{aligned}\\
\label{eq:y-iter-dsgd}
&\begin{aligned}
\vy^\kp&=b^\ko(\vy^\ko-\gamma\ve^\ko)+\\
&\quad(1-b^\ko)\vP^\ko(\vz^\ko-\gamma\ve^\ko).
\end{aligned}
\end{align}

\textbf{Extension to 1-port system.} \Ours\ listed in Algorithm \ref{alg:dsgd-ceca} is designed for 2-port communication system. But it can be easily extended to  1-port system with even $n$. The algorithm recursions are almost the same, except that in each iteration, the communication matrix $\vP^\ko$ is sampled as $\vQ^\ro$ in \eqref{eq:1p-gossip-matrix}. 

\textbf{Implementation.} The computation cost that incurs the highest expense in \Ours, primarily lies in the calculation of gradients. It is worth noting that, in each iteration, every agent $i$ simply needs to compute a single stochastic gradient either at the current iterate $\bvx_i^\ko$ or $\bvy_i^\ko$. Furthermore, the iteration updates given by equations \eqref{eq:x-iter-dsgd} and \eqref{eq:y-iter-dsgd} involve only additions and scaling operations. Consequently, the additional computational burden introduced by \Ours\ is relatively minor when compared to vanilla DSGD.

\subsection{Convergence analysis}

\textbf{CECA breaks double-stochasticity.} To analyze \Ours\ for both 1-port and 2-port communication models, we introduce the following model mixing matrix $\vW\in\bbR^{2n\times 2n}$ and gradient mixing matrix $\vW_g\in\bbR^{2n\times 2n}$. In iteration $k$, we calculate $r=k\Mod{\tau}$. If $\delta_r=1$, we have
\begingroup
\allowdisplaybreaks
\begin{equation}
\label{eq:W-Wg-d1}
\begin{split}
\vW^\ko&=\left[
\begin{array}{cc}
    \frac{1}{2}\vI_n+\frac{1}{2}\vP^\ko & \vzero_n \\
    \frac{n_r+1}{2n_r+1}\vP^\ko & \frac{n_r}{2n_r+1}\vI_n
\end{array}
\right], \\
\vW^\ko_g&=\left[
\begin{array}{cc}
    \frac{1}{2}\vI_n+\frac{1}{2}\vP^\ko & \vzero_n \\
    \frac{n_r}{2n_r+1}\vI_n+\frac{n_r+1}{2n_r+1}\vP^\ko & \vzero_n
\end{array}
\right].
\end{split}
\end{equation}
If $\delta_r=0$, we have
\begin{equation}
\label{eq:W-Wg-d0}
\begin{split}
\vW^\ko&=\left[
\begin{array}{cc}
    \frac{n_r+1}{2n_r+1}\vI_n & \frac{n_r}{2n_r+1}\vP^\ko \\
    \vzero_n & \frac{1}{2}\vI_n+\frac{1}{2}\vP^\ko
\end{array}
\right],  \\
\vW^\ko_g&=\left[
\begin{array}{cc}
    \vzero_n & \frac{n_r+1}{2n_r+1}\vI_n+\frac{n_r}{2n_r+1}\vP^\ko \\
    \vzero_n & \frac{1}{2}\vI_n+\frac{1}{2}\vP^\ko
\end{array}
\right].
\end{split}
\end{equation}
\endgroup
With mixing matrices $\vW$ and $\vW_g$, the update in \eqref{eq:x-iter-dsgd} and \eqref{eq:y-iter-dsgd} can be simply written as
\begin{equation}
\label{eq:iter-with-trans-matrix}
\left[
\begin{array}{c}
     \vx^\kp  \\
     \vy^\kp 
\end{array}
\right]
=\vW^\ko
\left[
\begin{array}{c}
     \vx^\ko  \\
     \vy^\ko 
\end{array}
\right]
-\gamma \vW^\ko_g
\left[
\begin{array}{c}
     \vg^\ko  \\
     \vh^\ko 
\end{array}
\right].
\end{equation}
Vanilla DSGD typically requires mixing matrix be doubly stochastic, i.e., each row and column sum equals $1$. This adorable property will enable the algorithm to converge to a consensus and correct solution. However, it is observed in \eqref{eq:W-Wg-d1} or \eqref{eq:W-Wg-d0} that neither $\vW$ nor $\vW_g$ is column-stochastic. This will bring fundamental challenges to establish convergence guarantees for \Ours. 

\textbf{Favorable properties of $\vW$ and $\vW_g$.} We next establish several fundamental properties of $\vW$ and $\vW_g$. To this end, we introduce the \textit{mixing matrix family}
$\mathscr{F}$ as follows,
\begin{multline}
\label{eq:F-family}
\mathscr{F}=\Big\{\vW=\left[
\begin{array}{cc}
    c \vW_{1,1} & (1-c)\vW_{1,2} \\
    d\vW_{2,1} & (1-d)\vW_{2,2} 
\end{array}
\right]\in\bbR^{2n\times 2n},\\\mbox{ where } 0\leq c,d\leq 1, \vW_{i,j}\in\bbR^{n\times n} \textup{ doubly stochastic},\\ i,j\in\{1,2\}\Big\}
\end{multline}
We summarize several properties for any $\vW \in \mathscr{F}$. The proofs are in Appendix \ref{append:trans-matrix-family}. 
\begin{lemma}
\label{lem:matrix_family}
The matrix family $\mathscr{F}$ satisfies the following properties:\vspace{-3mm}
\begin{itemize}[leftmargin=10pt]
    \item The matrix family $\mathscr{F}$ is a convex subset of the row stochastic matrices. 
    \item For any $\vW\in\mathscr{F}$, it holds that $\|\vW\|\leq \sqrt 2$.
    \item For any $\vW\in\mathscr{F}$, it holds that 
\begin{equation*}
\vW
\left[
\begin{array}{cc}
    \frac{1}{n}\bbone_n\bbone_n^\top & \vzero_n \\
    \vzero_n & \frac{1}{n}\bbone_n\bbone_n^\top
\end{array}
\right]
=
\left[
\begin{array}{cc}
    \frac{1}{n}\bbone_n\bbone_n^\top & \vzero_n \\
    \vzero_n & \frac{1}{n}\bbone_n\bbone_n^\top
\end{array}
\right]
\vW,
\end{equation*}
{where $\bbone_n$ is the all-ones vector in $\bbR^n$.}
\end{itemize}
\end{lemma}
\begin{lemma}\label{lm-fundamental}
    {In Algorithm \ref{alg:dsgd-ceca},} when the communication matrix $\vP^\ko$ is sampled from \eqref{eq:2p-gossip-matrix-d1} or \eqref{eq:2p-gossip-matrix-d0} {according to $\delta_r$}, and the mixing matrix $\vW^\ko$ is sampled from \eqref{eq:W-Wg-d1} or \eqref{eq:W-Wg-d0} {according to $\delta_r$}. The product of  matrices $\vW^\ko$ satisfies
    \begin{align}
     \label{eq:matrix-consensus2-main}
\prod_{k=0}^{t}\vW^\ko=
\left[
\begin{array}{cc}
    \frac{1}{n}\bbone_n\bbone_n^\top & \vzero_n \vspace{1mm}\\
    \frac{1}{n}\bbone_n\bbone_n^\top & \vzero_n
\end{array}\right], \quad \forall t\geq \tau.
    \end{align}
\end{lemma}
\begin{remark} Lemma \ref{lem:matrix_family} and Lemma \ref{lm-fundamental} are fundamental to establish the convergence property of \Ours. {Lemma \ref{lem:matrix_family} implies that the gossip matrix in \Ours, while not doubly stochastic, belongs to a family that shares many similarities with the doubly stochastic matrix family. These properties help break the doubly stochastic constraint in the standard DSGD analysis framework. Lemma \ref{lm-fundamental} essentially states that, while not being column-stochastic, the special structure of the gossip matrix} as constructed in \eqref{eq:W-Wg-d1} or \eqref{eq:W-Wg-d0} can still enable global average after multiplying with more than $(\tau+1)$ consecutive $\vW^\ko$. {Both models $\bvx$ and $\bvy$ can achieve the global average, which extends beyond the findings of \cite{bar1993optimal} that focused solely on the consensus of the $\bvx$ model. }
\end{remark}

\textbf{Convergence property.} We finally establish the convergence theorem of Algorithm \ref{alg:dsgd-ceca}. {We let $\bar\vx^\ko=\avein\bvx^\ko_i\in\bbR^d$ be the average of all local model.}

\begin{theorem}[\sc Convergence property]
\label{thm:conv_property}
Suppose Assumptions \ref{ass:Lip}-\ref{ass:data-hetero} hold,  {and we conduct global averaging in the first $\tau$ steps so that} $\bvx_i^\ko = \bar{\vx}^\ko$, $\bvy_i^\ko = \bar{\vy}^\ko$ for $0 \le k < \tau$. {Starting from the $(\tau+1)$th step, we perform DSGD-CECA iterations \eqref{eq:x-iter-dsgd}, \eqref{eq:y-iter-dsgd}.} If  $\gamma$ satisfies 
\begin{align*}
\gamma = \frac{1}{\left( \frac{2n\Delta}{L\sigma^2(T+1)}\right)^{-\frac{1}{2}} + \left( \frac{\Delta}{24L^2\tau^2(\sigma^2+2b^2)(T+1)}\right)^{-\frac{1}{3}} + 8\tau L}
\end{align*}
where $\Delta = \bbE f(\bar\vx^{0}) - f^\star$, then DSGD-CECA converges at
\begin{align}
\label{znw9237-0-present}
\frac{1}{T+1}\sum_{k=0}^T \bbE\|\nabla f(\bar\vx^\ko)\|^2 
\le 16 \left( \frac{\Delta L \sigma^2}{n(T+1)} \right)^{\frac{1}{2}} \nonumber \\
\quad + 24 \left(\frac{\Delta^2 L^2 \tau^2(\sigma^2 + 2b^2)}{(T+1)^2}\right)^{\frac{1}{3}} + \frac{32\tau \Delta L}{T+1}.
\end{align}
with any $n$ when utilizing the 2-port communication model, or with any even $n$ when utilizing the 1-port model. 
\end{theorem}
{The proof of Theorem \ref{thm:conv_property} can be referred to Appendix \ref{append:conv_thm}.}
\begin{remark}
Based on the convergence rate in \eqref{znw9237-0-present}, we can derive that when $T = O(n^3\log^4_2(n))$, {the linear speedup term $O(1/\sqrt{nT})$ dominates the other two terms $O(\tau^\frac{2}{3}/T^\frac{2}{3})$ and $O(\tau/T)$ up to a constant scalar. This linear speedup term dominates} the convergence rate. This implies \Ours\ has $O(n^3\log^4_2(n))$ transient iterations.
\end{remark}

\section{Numerical Experiments}
In this section, we validate the previous theoretical results via numerical experiments. 
First, we show CECA-1P indeed achieves the global consensus in finite iterations over a variety of choices of the number of nodes. Next, we examine the performance of DSGD-CECA and compare it with many other popular SOTA algorithms on a standard convex task. Lastly, we apply the DSGD-CECA on the deep learning setting to show it still achieves good performance in train loss and test accuracy with respect to the iterations and communicated data. The  codes used to generate the figures in this section are available in the github\footnote{\url{https://github.com/kexinjinnn/DSGD-CECA}}.

\textbf{Finite-time exact consensus convergence.}
We examine the convergence rate of CECA-1P over different network sizes $n$. In each experiment, we initialize a random vector {$\bvx_i^{(0)},\bvy_i^{(0)}$} on each node $i$, and obtain  {$\bvx_i^\ko,\bvy_i^\ko$} by applying the communication topology. The residue $\sum_i\|\bvx_i^\ko-\bvx^\star\|$ is calculated at each iteration $k$, where $\bvx^\star$ is the global average of all initial $\bvx_i^{(0)}$. From Fig.~\ref{fig:consensusN}, we observe the results coincide with the proved theorem. Especially, the number of iterations of CECA-1P to achieve exact global average is $\lceil \log_2(n) \rceil$ as the theorem predicted.

Next, we compare CECA with other popular topologies. We set the network size to $n=130$ and $n=1026$, respectively (both are not the power of 2 but close to it). Results are averaged over $3$ independent random runs. In Fig.~\ref{fig:consensusComp}, we observe CECA achieves global average with a finite number of iterations, whereas the others do not.

\begin{figure}[ht]
\vskip -0.05in
\begin{center}
\centerline{\includegraphics[width=0.7\columnwidth]{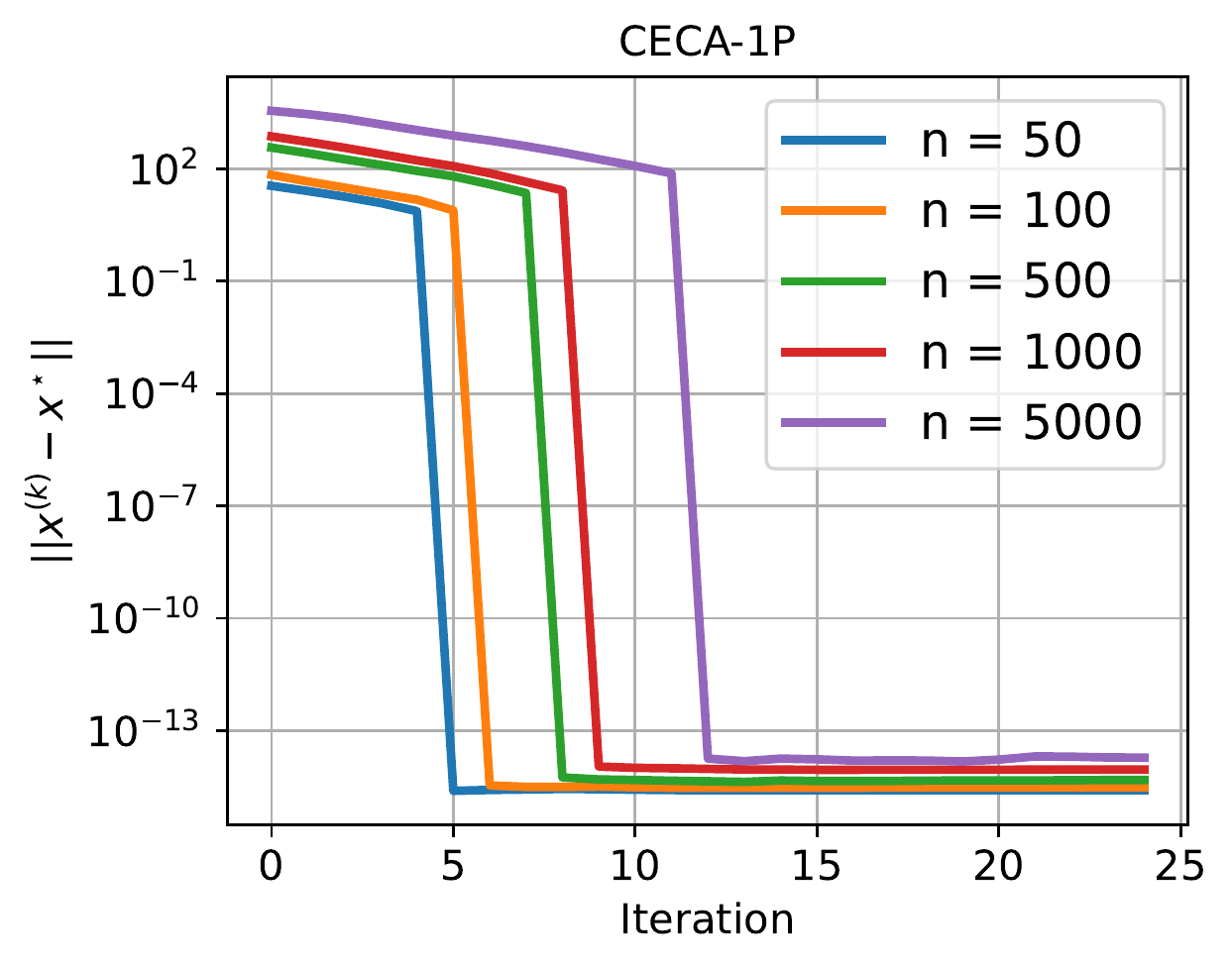}}
\vskip -0.15in
\caption{The CECA-1P can achieve finite-time consensus convergence for different network sizes.}
\label{fig:consensusN}
\end{center}
\vskip -0.15in
\end{figure}

\begin{figure}[ht]
\vskip 0.05in
\begin{center}
\centerline{\includegraphics[width=0.7\columnwidth]{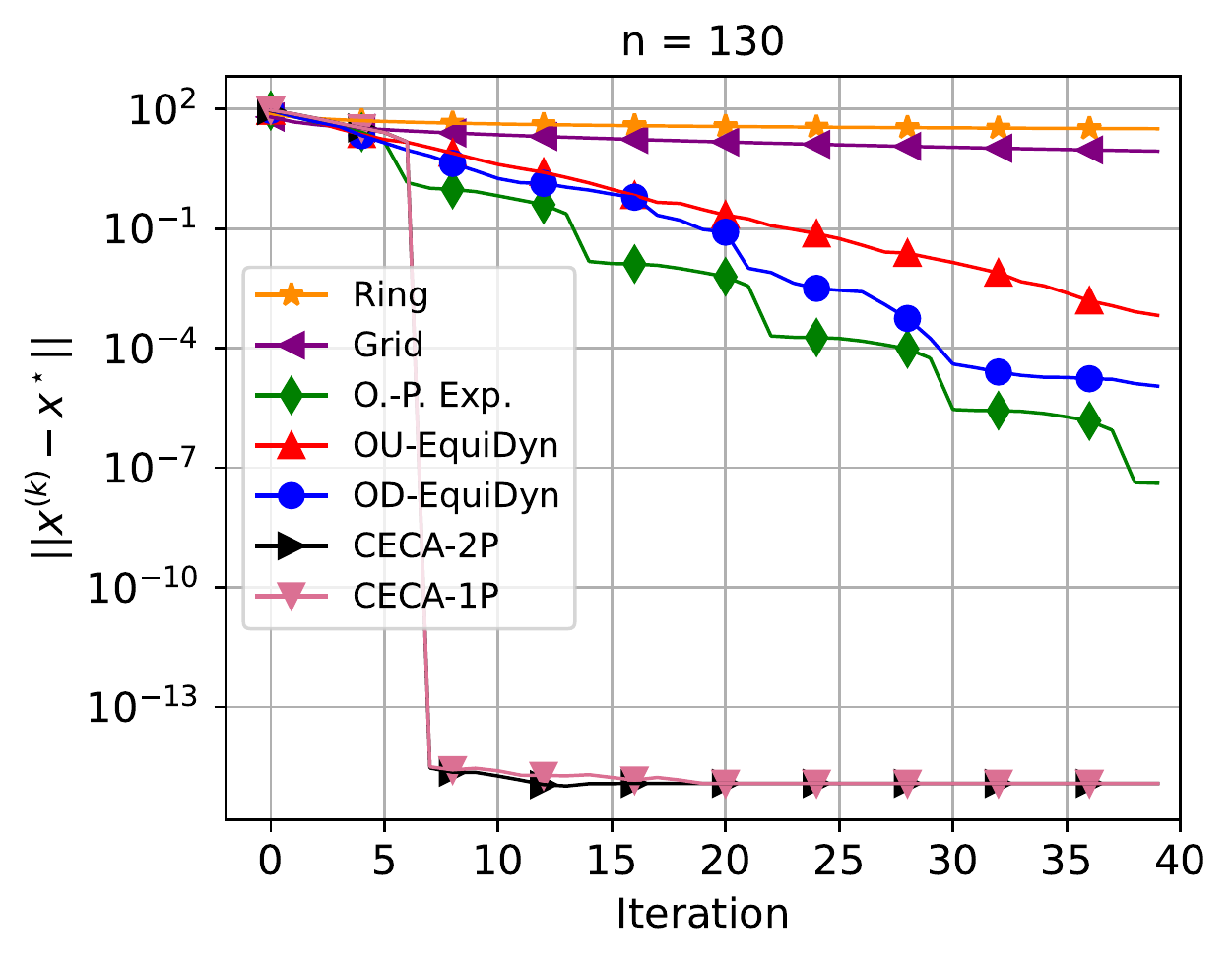}}
\centerline{\includegraphics[width=0.7\columnwidth]{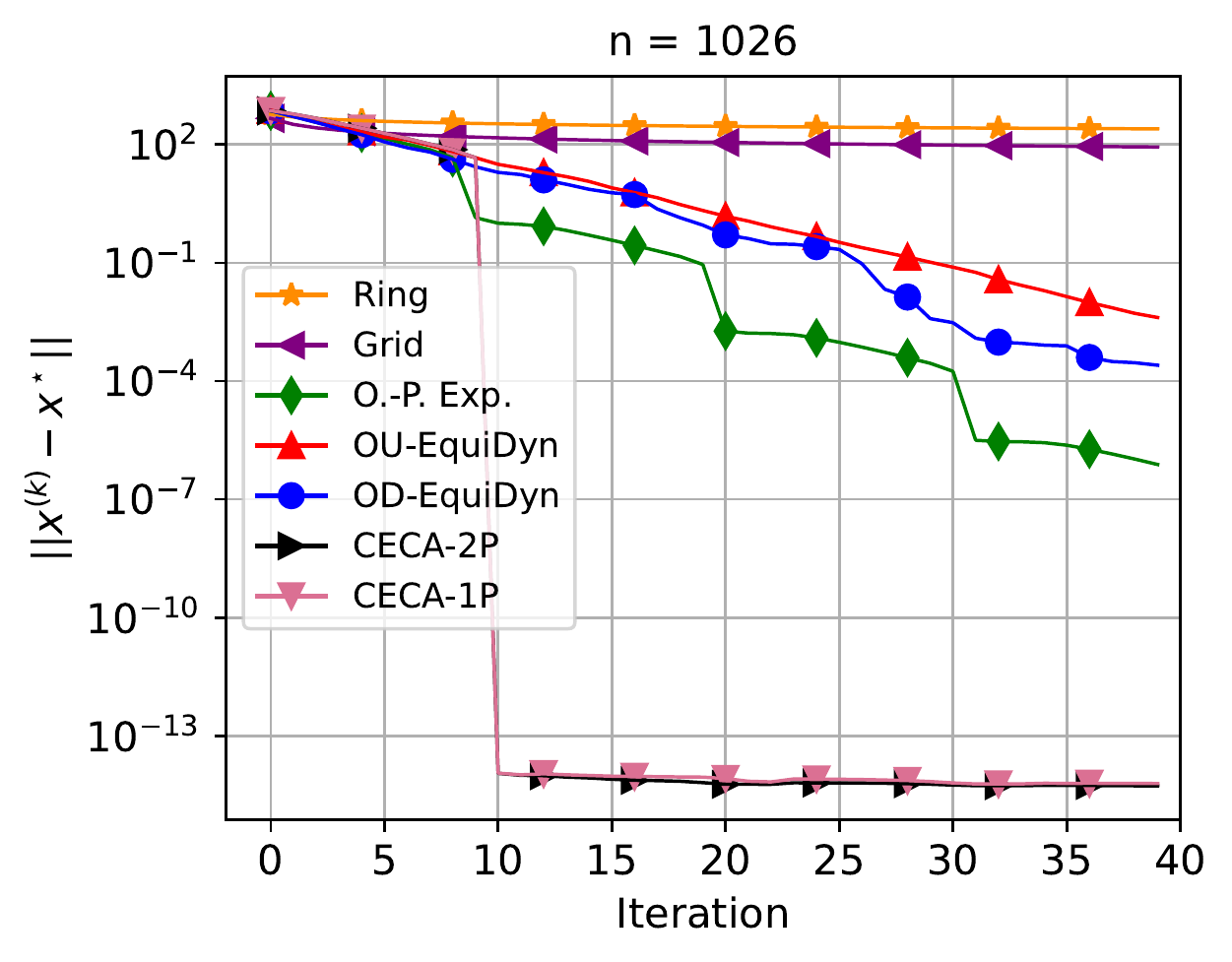}}
\vskip -0.15in
\caption{The consensus rates of CECA-2P and CECA-1P are faster than those of the other $\Theta(1)$-neighbor topologies.}
\label{fig:consensusComp}
\end{center}
\vskip -0.3in
\end{figure}

\textbf{DSGD: least-squares problem.}
We examine DSGD-CECA by solving the distributed least square problem in which each $f_i(\bvx) := (1/2)\|\vA_i \bvx - \bvb_i\|^2$ where $i\in \{1,2,\ldots, n\}$, $\bvx\in \mathbb{R}^d$ and $\vA_i \in \mathbb{R}^{N\times d}$. We generate $\vA_i$ from $\mathcal{N}(0, I)$. The measurement $\bvb_i$ is generated by $\bvb_i = \vA_i \bvx^\star + \bvv_i$ with a given $\bvx^\star \in \mathbb{R}^d$ where $\bvv_i \sim \mathcal{N}(0,\sigma_s^2 \vI)$. Each node will generate a stochastic gradient via $\widehat{\nabla f}_i(\bvx) = {\nabla f}_i(\bvx) + \vepsilon_i$ at each iteration, where $\vepsilon_i \sim \mathcal{N}(0,\sigma_n^2 \vI)$ is the noise level of SGD. In the simulation, we set the size of the network $n=258$, $d=10$, $N=50$, $\sigma_s=0.1$, and $\sigma_n=5$. We set the initial learning rate to be $0.02$ to all algorithms. Then, every 20 iterations the learning rate decays by a factor $1.5$. The results are averaged over $5$ independent random experiments. Fig.~\ref{fig:DSGD_convex} depicts the performance of each algorithm. It is observed that DSGD-CECA achieves the best convergence performance.

\begin{figure}[ht]
\begin{center}
\centerline{\includegraphics[width=0.72\columnwidth]{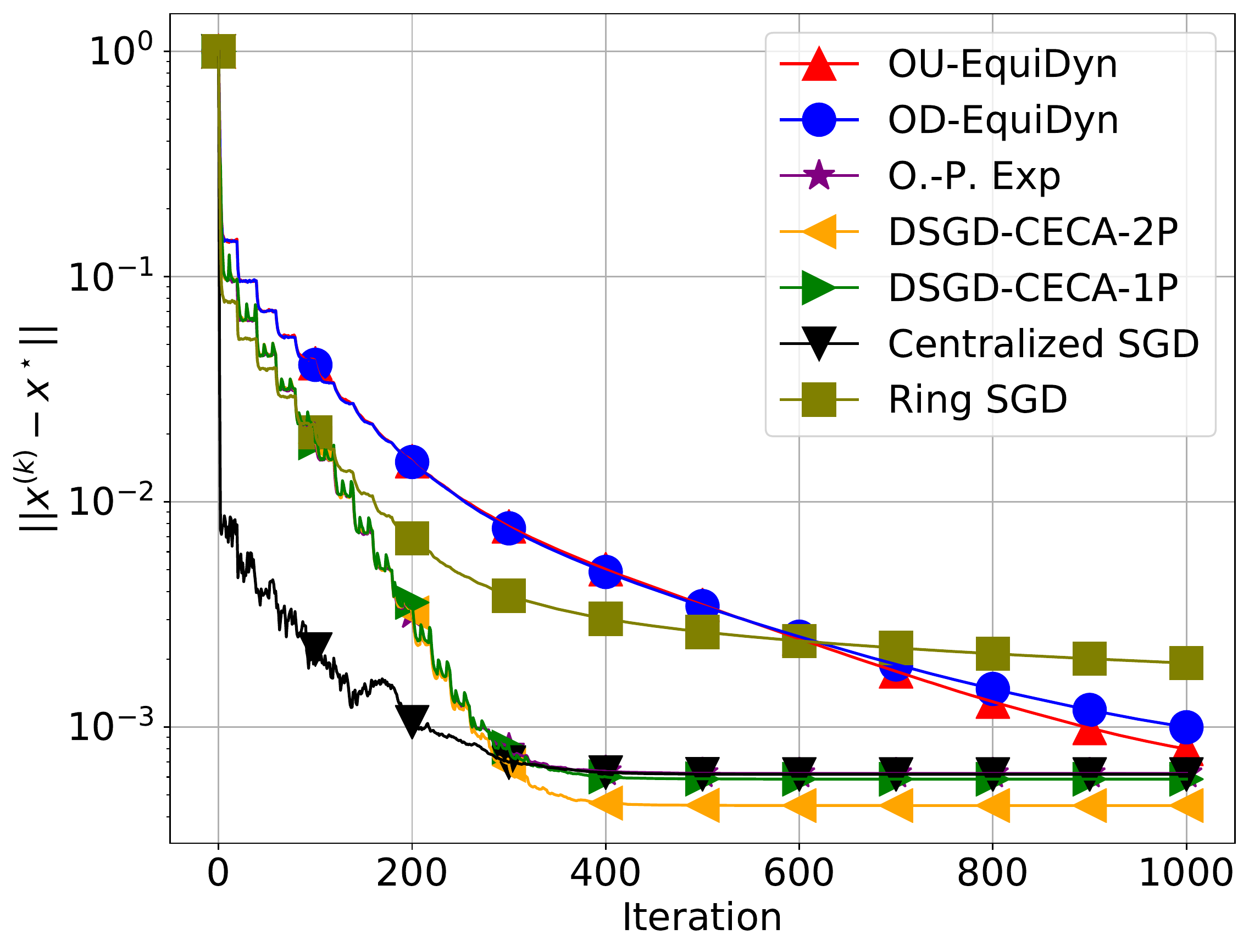}}
\vskip -0.1in
\caption{Performance comparison between stochastic decentralized algorithms using various effective topologies.}
\label{fig:DSGD_convex}
\end{center}
\vskip -0.2in
\end{figure}

\textbf{DSGD: deep learning.} We apply DSGD-CECA-2P to solve the image classification task with CNN over \textbf{MNIST} dataset \cite{lecun2010mnist}. As for the implementation of decentralized parallel training, we utilize BlueFog library \cite{bluefog} in a cluster of 17 NVIDIA GeForce RTX 2080 GPUs.  The network contains two convolutional layers with max pooling and ReLu and two feed-forward layers. The local batch size is set to 64. The learning rate is $0.3$ for DSGD-CECA-2P with no momentum and $0.1$ for other algorithms with momentum $0.5$. 
The results are obtained by averaging
over 3 independent random experiments. Fig.~\ref{fig:mnist} illustrates the training loss and test accuracy curves, while Table \ref{table:deep_learning_mnist} provides the corresponding numerical values. The results indicate that DSGD-CECA-2P outperforms other decentralized algorithms, exhibiting slightly better training loss and test accuracy.
Fig.~\ref{fig:losssvsdata} depicts the performance of different algorithms in terms of data communicated. The data communicated is calculated based on the total length of the vectors that one node sends and receives. If different nodes have different values, we choose the one with the largest value since it is synchronized style algorithm. The figure implies that one-peer decentralized algorithms, including DSGD-CECA and O.-P.-Exp., outperform centralized SGD significantly. 

We also provide additional experiments on \textbf{CIFAR-10} dataset \cite{Krizhevsky09} in Appendix \ref{appendix: exp}. \vskip -0.1in

\begin{figure}[ht]
\vskip -0in
\begin{center}
\centerline{\includegraphics[width=0.6\columnwidth]{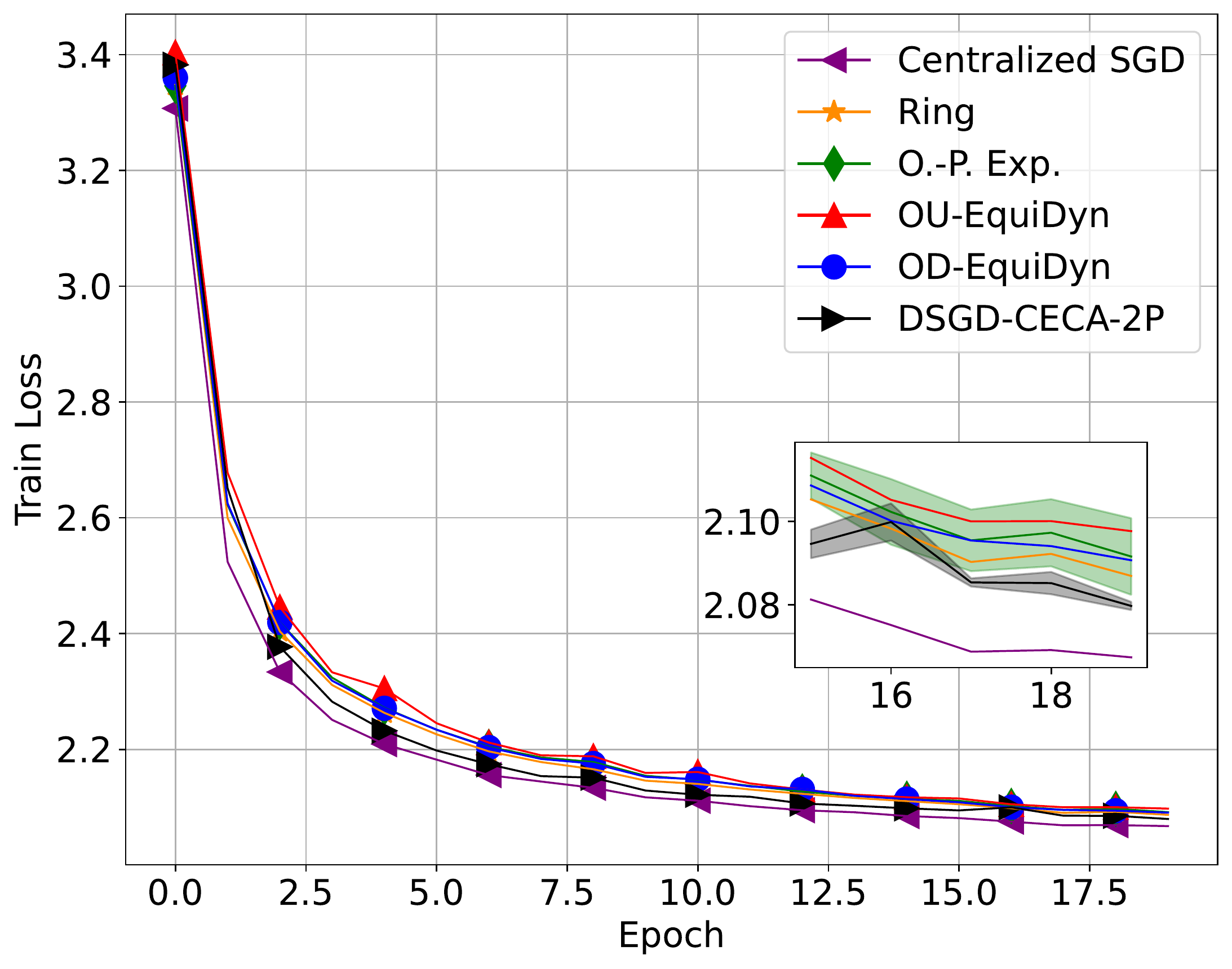}}
\centerline{\includegraphics[width=0.6\columnwidth]{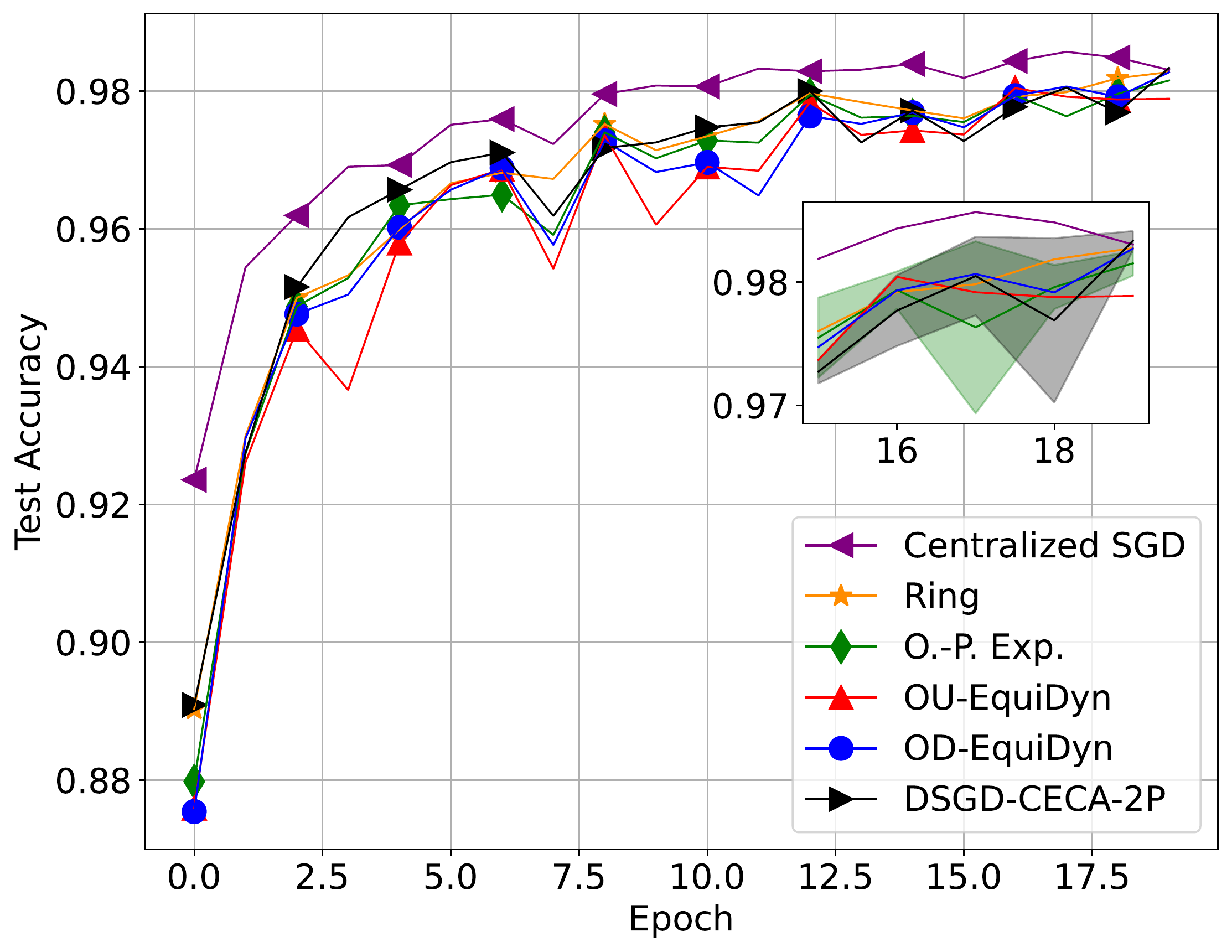}}
\vskip -0.1in
\caption{Train loss and test accuracy of different DSGD algorithms for CNN on MNIST. The solid curve indicates the average, while the shaded area indicates the deviation.}
\label{fig:mnist}
\end{center}
\vskip -0.1in
\end{figure}

\begin{figure}[ht]
\begin{center}
\centerline{\includegraphics[width=0.7\columnwidth]{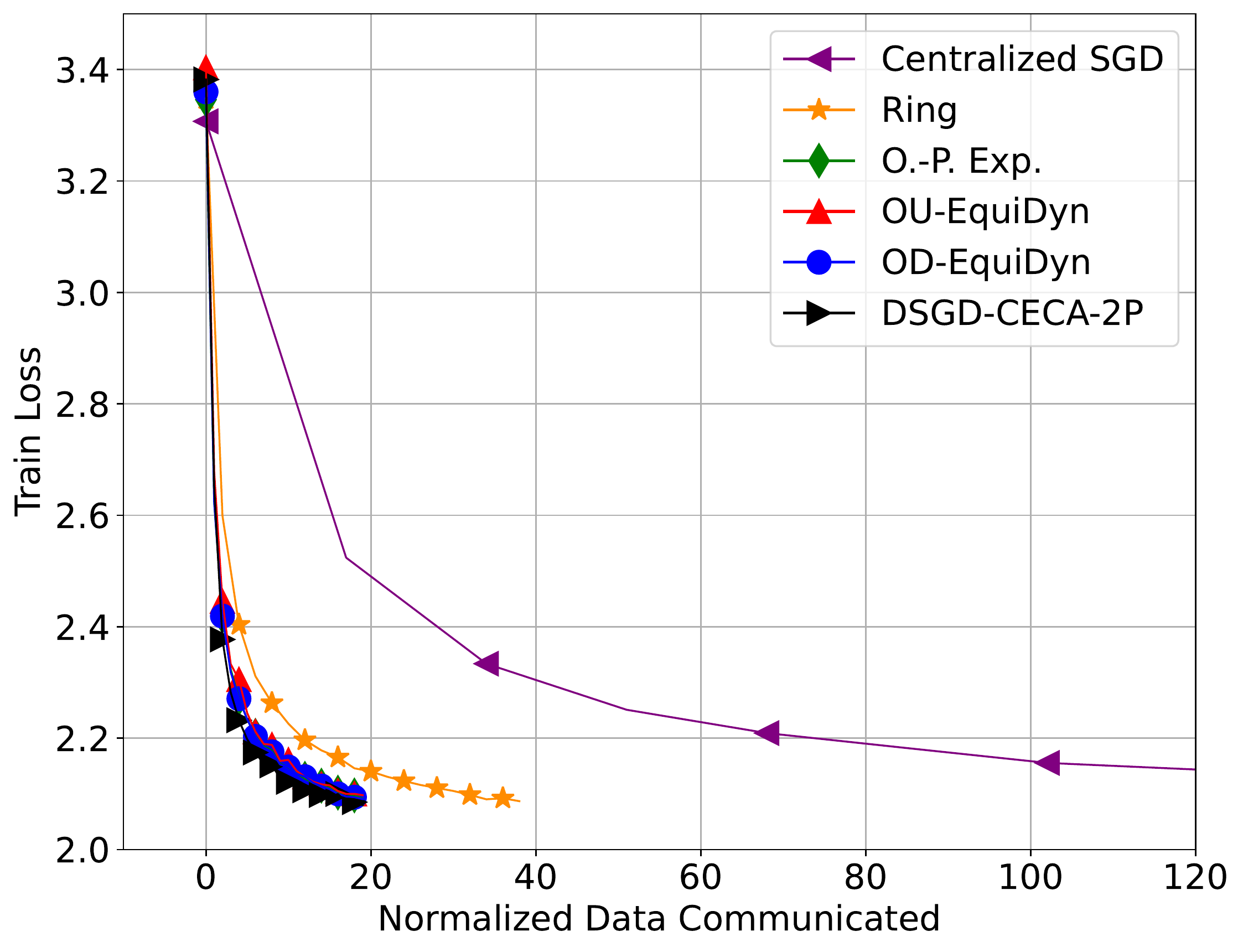}}\vskip -0.1in
\caption{Performance of different algorithms over MNIST dataset in terms of data communicated.}
\label{fig:losssvsdata}
\end{center}
\vskip -0.25in
\end{figure}

\setlength{\tabcolsep}{4pt}
\begin{table}[h!]
    \centering 
    \caption{\small Comparison of train loss and test accuracy(\%) with different topologies over MNIST dataset.}
    \vspace{2mm}
	\begin{tabular}{rccccc}
		\toprule
		&\textbf{Topology} &  \textbf{Train Loss}  & \textbf{Test Acc.} &\\ \midrule
		&Centralized SGD &   \textbf{2.079}  &  98.34  & \\ \hdashline
		&Ring        &    2.090      &  98.32  &   \\ 
		&O.-P. Exp.    &     2.091      & 98.33   &  \\  
		&OD-EquiDyn  &     2.090 &  98.36  & \\ 
		&OU-EquiDyn   &    2.091       &   98.03    &    \\ 	
            &DSGD-CECA-2P   &    \textbf{2.083}       &   \textbf{98.50} &  \\
		\bottomrule
	\end{tabular}
	\label{table:deep_learning_mnist}
\end{table}

\section{Conclusion}
In this paper, we propose a novel decentralized stochastic gradient descent algorithm, named DSGD-CECA. This algorithm consists of two versions: DSGD-CECA-1P and DSGD-CECA-2P, designed for the 1-port and 2-port message-passing models, respectively. The convergence rates of both versions are theoretically analyzed for non-convex stochastic optimization. The results demonstrate that, even at the minimal communication cost per iteration, the total number of iterations and the transient iterations are comparable to the state-of-the-art methods. Notably, the proposed methods are applicable to any number of agents, significantly relaxing the previous restriction of the power of two. Furthermore, empirical experiments validate the efficiency of DSGD-CECA in comparison to other DSGD algorithms.

\bibliography{example_paper}
\bibliographystyle{icml2023}

\newpage
\appendix
\onecolumn
\section{Notations}
\label{append:notations}
We introduce the following notations to simplify analysis.
\begin{itemize}
\item $\vx^\ko=[(\bvx^\ko_1)^\top;(\bvx^\ko_2)^\top;\cdots;(\bvx^\ko_n)^\top]\in \bbR^{n\times d}$ and $\bar\vx^\ko=\avein\bvx^\ko_i\in\bbR^d$.
\item $\vy^\ko=[(\bvy^\ko_1)^\top;(\bvy^\ko_2)^\top;\cdots;(\bvy^\ko_n)^\top]\in \bbR^{n\times d}$ and $\bar\vy^\ko=\avein\bvy^\ko_i\in\bbR^d$.
\item $\vz^\ko=[(\bvz^\ko_1)^\top;(\bvz^\ko_2)^\top;\cdots;(\bvz^\ko_n)^\top]\in \bbR^{n\times d}$ and $\bar\vz^\ko=\avein\bvz^\ko_i\in\bbR^d$.
\item $\vW^\ko=[w_{i,j}^\ko]\in\bbR^{2n\times 2n}$.
\item $\bbone_n=[1;1;\cdots;1]\in\bbR^n$.
\item $\nabla F(\vx)=[(\nabla f_1(\bvx_1))^\top;(\nabla f_2(\bvx_2))^\top;\cdots;(\nabla f_n(\bvx_n))^\top]\in \bbR^{n\times d}$ for any  $\vx=[\bvx_1^\top;\bvx_2^\top;\cdots;\bvx_n^\top]\in\bbR^{n\times d}$.
\item $\nabla F(\vx;\vxi)=[(\nabla F_1(\bvx_1;\xi_1))^\top;(\nabla F_2(\bvx_2;\xi_2))^\top;\cdots;(\nabla F_n(\bvx_n;\xi_n))^\top]\in \bbR^{n\times d}$.
\item $\vg^{(k)}=\nabla F(\vx^\ko;\vxi^\ko)$ { and $\bar{\vg}^\ko=\frac{1}{n}\sum_{i=1}^n\nabla F(\bvx^\ko_i;\vxi^\ko_i)$}.
\item $\vh^\ko=\nabla F(\vy^\ko;\vxi^\ko)$ { and $\bar{\vh}^\ko=\frac{1}{n}\sum_{i=1}^n\nabla F(\bvy^\ko_i;\vxi^\ko_i)$}.
\item { 
$\ve^\ko=\nabla F(\vz^\ko;\vxi^\ko)$ and $\bar{\ve}^\ko=\frac{1}{n}\sum_{i=1}^n\nabla F(\bvz^\ko_i;\vxi^\ko_i)$. }
\item Given two matrices $\vx, \vy \in \mathbb{R}^{n\times d}$, we define the inner product $\langle \vx, \vy \rangle = \mathrm{tr}(\vx^\top \vy)$, the Frobenius norm $\|\vx\|_F^2 = \langle \vx, \vx \rangle$. %
\item Given any vector $\bvx \in \mathbb{R}^d$, we let $\|\bvx\|$ be its $\ell_2$ norm.
\item Given a sequence of matrices $\{\vA^{(k)}\}_{k=i}^j$ with $j\ge i$,  we let $\prod_{k=i}^j \vA^{(k)}= \vA^{(j)} \vA^{(j-1)} \cdots \vA^{(i+1)} \vA^{(i)}$. If $j < i$, we let $\prod_{k=i}^j \vA^{(k)} = \vI $.
\end{itemize}

\section{Optimal dynamic topology DSGD}
\subsection{Supplementary materials on CECA}
\label{append:CECA-allreduce-supp}
\subsubsection{CECA for the 2-port message passing system}
\label{append:CECA-2-port}
In this section, we present the CECA in a detailed way. The pseudo code of the CECA as in Algorithm \ref{alg:opt-2p}.

Before the iterations start, we calculate $\tau=\lceil \log_2 n\rceil$. We compute $\delta_r, n_r, r=0,1,2,\ldots,\tau-1$ as in \eqref{eq:bin-n-1} and \eqref{eq:iter-nr}.

The method is initialized before the round 0 with each agent $i$ holding information $I^{(0)}_i=u_i, J^{(0)}_i=0$. We use $I^\ro_i$ and $J^\ro_i$ to denote agent $i$'s information right before the $r^{\text{th}}$ round, $r=0,1,\ldots,\tau-1$. In the $r^{\text{th}}$ round, if $\delta_r=1$, agent $i$ sends information $I^\ro_i$ to agent $i+n_r+1$ and receives information $I^\ro_{i-n_r-1}$ from agent $i-n_r-1$; if $\delta_r=0$, agent $i$ sends information $J^\ro_i$ to agent $i+n_r$ and receives information $J^\ro_{i-n_r}$ from agent $i-n_r$.

After receiving the information, each agent $i$ updates information $I^\ro$ and $J^\ro$ with the received information. {We switch over cases where $\delta_r$ is either 0 or 1.} If $\delta_r=1$,
\begin{align*}
&I^{(r+1)}_i=\frac{1}{2}I^\ro_i+\frac{1}{2}I^\ro_{i-n_r-1},\\ 
&J^{(r+1)}_i=\frac{n_r}{2n_r+1}J^\ro_i+\frac{n_r+1}{2n_r+1}I^\ro_{i-n_r-1}.
\end{align*}
If $\delta_r=0$,
\begin{align*}
I^{(r+1)}_i=\frac{n_r+1}{2n_r+1}I^\ro_i+\frac{n_r}{2n_r+1}J^\ro_{i-n_r},\quad J^{(r+1)}_i=\frac{1}{2}J^\ro_i+\frac{1}{2}J^\ro_{i-n_r}.
\end{align*}
In either case $\delta_r=0$ or $\delta_r=1$, we have 
\begin{align*}
I^{(r+1)}_i=\frac{1}{n_{r+1}+1}\sum_{j=0}^{n_{r+1}}u_{i-j},\quad
J^{(r+1)}_i=\frac{1}{n_{r+1}}\sum_{j=1}^{n_{r+1}}u_{i-j}.
\end{align*}
by induction. After $\tau$ rounds of communication, 
$$
I^{(\tau)}_i=\frac{1}{n}\sum_{j=1}^n u_j,~ J^{(\tau)}_i=\frac{1}{n-1}\sum_{j\neq i} u_j.
$$

\begin{algorithm}[tb]
   \caption{CECA for the 2-port system}
   \label{alg:opt-2p}
\begin{algorithmic}
   \STATE {\bfseries Input:} $n$ agents: each agent $i$ has data $u_i, i=1,2,\ldots,n$.
   \STATE Set $\tau=\lceil \log_2n \rceil$ (total rounds of message passing);
   \STATE Represent $n-1$ with the binary sequence $(\delta_0\,\delta_1\,\cdots\, \delta_{\tau-1} )_2$;
   \STATE Initialize $n_0=0$ and, for $i=1,\dots,n$, $I^{(0)}_i=u_i$ and $ J^{(0)}_i=0$;
   \FOR{$r=0,1,2,\ldots,\tau-1$}
   \IF{$\delta_r=1$}
   \STATE Agent $i$ sends information $I^\ro_i$ to agent $i+n_r+1$ and receives information $I^\ro_{i-n_r-1}$ from agent $i-n_r-1$;
   \STATE $
    I^{(r+1)}_i=\frac{1}{2}I^\ro_i+\frac{1}{2}I^\ro_{i-n_r-1}$;
   \STATE$
    J^{(r+1)}_i=\frac{n_r}{2n_r+1}J^\ro_i+\frac{n_r+1}{2n_r+1}I^\ro_{i-n_r-1}
    $;
   \ELSIF{$\delta_r=0$}
   \STATE Agent $i$ sends information $J^\ro_i$ to agent $i+n_r$ and receives information $J^\ro_{i-n_r}$ from agent $i-n_r$;
   \STATE $
    I^{(r+1)}_i=\frac{n_r+1}{2n_r+1}I^\ro_i+\frac{n_r}{2n_r+1}J^\ro_{i-n_r}$;
   \STATE$
    J^{(r+1)}_i=\frac{1}{2}J^\ro_i+\frac{1}{2}J^\ro_{i-n_r}
    $;
   \ENDIF 
   \STATE $n_{r+1}=2n_r+\delta_r$;
   \ENDFOR
   \STATE \bfseries{return} $I^{(\tau)}_i, i=1,2,\ldots,n$, which equal $\frac{1}{n}(u_1+\dots+u_n)$.
\end{algorithmic}
\end{algorithm}

\textbf{Example: }We consider the case where the number of agents $n=6$. To make consensus among agents, we need $\tau=\lceil \log_2(6)\rceil=3$ rounds. The binary representation of $n-1=5$ is
$$
5=(1\,0\,1)_2.
$$
According to \eqref{eq:bin-n-1}, we assign $\delta_0=1$ as the most significant (left most) digit of the binary representation. {Besides, we assign $\delta_1=0, \delta_2=1$. We calculate $n_0=0, n_1=1, n_2=2$. Besides, we let $\vI^\ro=[I^\ro_1;I^\ro_2;\cdots;I^\ro_6], \vJ^\ro=[J^\ro_1;J^\ro_2;\cdots;J^\ro_6]$ be the formal column vectors of $I^\ro_i$ and $J^\ro_i$. Both $I^\ro_i$ and $J^\ro_i$ are linear combinations of $u_1,u_2,\ldots,u_6$. We formally use the matrix-vector product 
$$
\left[
\begin{array}{cccc}
    a_{1,1} & a_{1,2} & \cdots & a_{1,6} \\
    a_{2,1} & a_{2,2} & \cdots & a_{2,6} \\
    \vdots  & \vdots  & \ddots & \vdots  \\
    a_{6,1} & a_{6,2} & \cdots & a_{6,6} \\
\end{array}
\right]
\left[
\begin{array}{c}
     u_1  \\
     u_2  \\
     \vdots\\
     u_6
\end{array}
\right]
=
\left[
\begin{array}{c}
     \sum_{j=1}^6 a_{1,j}u_j  \\
     \sum_{j=1}^6 a_{2,j}u_j  \\
     \vdots\\
     \sum_{j=1}^6 a_{6,j}u_j
\end{array}
\right]
$$
to be the representation of $\vI^\ro$ or $\vJ^\ro$.} At the very beginning, agent $i$ has information $I^{(0)}_i=u_i, J^{(0)}_i=0$. So,
$$
\vI^{(0)}
=
\left[
\begin{array}{cccccc}
    1 & 0 & 0 & 0 & 0 & 0 \\
    0 & 1 & 0 & 0 & 0 & 0 \\
    0 & 0 & 1 & 0 & 0 & 0 \\
    0 & 0 & 0 & 1 & 0 & 0 \\
    0 & 0 & 0 & 0 & 1 & 0 \\
    0 & 0 & 0 & 0 & 0 & 1 \\
\end{array}
\right]
\left[
\begin{array}{c}
     u_1  \\
     u_2  \\
     u_3  \\
     u_4  \\
     u_5  \\
     u_6
\end{array}
\right],\quad
\vJ^{(0)}
=
\left[
\begin{array}{cccccc}
    0 & 0 & 0 & 0 & 0 & 0 \\
    0 & 0 & 0 & 0 & 0 & 0 \\
    0 & 0 & 0 & 0 & 0 & 0 \\
    0 & 0 & 0 & 0 & 0 & 0 \\
    0 & 0 & 0 & 0 & 0 & 0 \\
    0 & 0 & 0 & 0 & 0 & 0 \\
\end{array}
\right]
\left[
\begin{array}{c}
     u_1  \\
     u_2  \\
     u_3  \\
     u_4  \\
     u_5  \\
     u_6
\end{array}
\right].
$$
In the $0^{\text{th}}$ round, $\delta_0=1$, agent $i$ sends $I^{(0)}_i=u_i$ to agent $i+1$ and receives $I^{(0)}_{i-1}=u_{i-1}$ from agent $i-1$. Agent $i$ averages the received information with $I^{(0)}_i, J^{(0)}_i$, respectively, and get  $I^{(1)}_i, J^{(1)}_i$ as follows,
\begingroup
\renewcommand{\arraystretch}{1.2} %
$$
\vI^{(1)}
=
\left[
\begin{array}{cccccc}
    \frac{1}{2} & 0 & 0 & 0 & 0 & \frac{1}{2} \\
    \frac{1}{2} & \frac{1}{2} & 0 & 0 & 0 & 0 \\
    0 & \frac{1}{2} & \frac{1}{2} & 0 & 0 & 0 \\
    0 & 0 & \frac{1}{2} & \frac{1}{2} & 0 & 0 \\
    0 & 0 & 0 & \frac{1}{2} & \frac{1}{2} & 0 \\
    0 & 0 & 0 & 0 & \frac{1}{2} & \frac{1}{2}
\end{array}
\right]
\left[
\begin{array}{c}
     u_1  \\
     u_2  \\
     u_3  \\
     u_4  \\
     u_5  \\
     u_6
\end{array}
\right],\quad
\vJ^{(1)}
=
\left[
\begin{array}{cccccc}
    0 & 0 & 0 & 0 & 0 & 1 \\
    1 & 0 & 0 & 0 & 0 & 0 \\
    0 & 1 & 0 & 0 & 0 & 0 \\
    0 & 0 & 1 & 0 & 0 & 0 \\
    0 & 0 & 0 & 1 & 0 & 0 \\
    0 & 0 & 0 & 0 & 1 & 0 
\end{array}
\right]
\left[
\begin{array}{c}
     u_1  \\
     u_2  \\
     u_3  \\
     u_4  \\
     u_5  \\
     u_6
\end{array}
\right].
$$
In the $1^{\textup{st}}$ round, $\delta_1=0$, agent $i$ sends $J^{(1)}_i=u_{i-1}$ to agent $i+1$ and receives $J^{(1)}_{i-1}=u_{i-2}$ from agent $i-1$. After averaging, we get
$$
\vI^{(2)}
=
\left[
\begin{array}{cccccc}
    \frac{1}{3} & 0 & 0 & 0 & \frac{1}{3} & \frac{1}{3} \\
    \frac{1}{3} & \frac{1}{3} & 0 & 0 & 0 & \frac{1}{3} \\
    \frac{1}{3} & \frac{1}{3} & \frac{1}{3} & 0 & 0 & 0 \\
    0 & \frac{1}{3} & \frac{1}{3} & \frac{1}{3} & 0 & 0 \\
    0 & 0 & \frac{1}{3} & \frac{1}{3} & \frac{1}{3} & 0 \\
    0 & 0 & 0 & \frac{1}{3} & \frac{1}{3} & \frac{1}{3} 
\end{array}
\right]
\left[
\begin{array}{c}
     u_1  \\
     u_2  \\
     u_3  \\
     u_4  \\
     u_5  \\
     u_6
\end{array}
\right],\quad
\vJ^{(2)}
=
\left[
\begin{array}{cccccc}
    0 & 0 & 0 & 0 & \frac{1}{2} & \frac{1}{2} \\
    \frac{1}{2} & 0 & 0 & 0 & 0 & \frac{1}{2} \\
    \frac{1}{2} & \frac{1}{2} & 0 & 0 & 0 & 0 \\
    0 & \frac{1}{2} & \frac{1}{2} & 0 & 0 & 0 \\
    0 & 0 & \frac{1}{2} & \frac{1}{2} & 0 & 0 \\
    0 & 0 & 0 & \frac{1}{2} & \frac{1}{2} & 0 
\end{array}
\right]
\left[
\begin{array}{c}
     u_1  \\
     u_2  \\
     u_3  \\
     u_4  \\
     u_5  \\
     u_6
\end{array}
\right].
$$
In the $2^{\textup{nd}}$ round, $\delta_2=1$, agent $i$ sends $I^{(2)}_i=\frac{1}{3}u_{i-2}+\frac{1}{3}u_{i-1}+\frac{1}{3}u_{i}$ to agent $i+3$, and receives $I^{(2)}_{i-3}=\frac{1}{3}u_{i-5}+\frac{1}{3}u_{i-4}+\frac{1}{3}u_{i-3}$ from agent $i-3$. After averaging the information, we get
$$
\vI^{(3)}
=
\left[
\begin{array}{cccccc}
    \frac{1}{6} & \frac{1}{6} & \frac{1}{6} & \frac{1}{6} & \frac{1}{6} & \frac{1}{6} \\
    \frac{1}{6} & \frac{1}{6} & \frac{1}{6} & \frac{1}{6} & \frac{1}{6} & \frac{1}{6} \\
    \frac{1}{6} & \frac{1}{6} & \frac{1}{6} & \frac{1}{6} & \frac{1}{6} & \frac{1}{6} \\
    \frac{1}{6} & \frac{1}{6} & \frac{1}{6} & \frac{1}{6} & \frac{1}{6} & \frac{1}{6} \\
    \frac{1}{6} & \frac{1}{6} & \frac{1}{6} & \frac{1}{6} & \frac{1}{6} & \frac{1}{6} \\
    \frac{1}{6} & \frac{1}{6} & \frac{1}{6} & \frac{1}{6} & \frac{1}{6} & \frac{1}{6}
\end{array}
\right]
\left[
\begin{array}{c}
     u_1  \\
     u_2  \\
     u_3  \\
     u_4  \\
     u_5  \\
     u_6
\end{array}
\right],\quad
\vJ^{(3)}
=
\left[
\begin{array}{cccccc}
    0 & \frac{1}{5} & \frac{1}{5} & \frac{1}{5} & \frac{1}{5} & \frac{1}{5} \\
    \frac{1}{5} & 0 & \frac{1}{5} & \frac{1}{5} & \frac{1}{5} & \frac{1}{5} \\
    \frac{1}{5} & \frac{1}{5} & 0 & \frac{1}{5} & \frac{1}{5} & \frac{1}{5} \\
    \frac{1}{5} & \frac{1}{5} & \frac{1}{5} & 0 & \frac{1}{5} & \frac{1}{5} \\
    \frac{1}{5} & \frac{1}{5} & \frac{1}{5} & \frac{1}{5} & 0 & \frac{1}{5} \\
    \frac{1}{5} & \frac{1}{5} & \frac{1}{5} & \frac{1}{5} & \frac{1}{5} & 0
\end{array}
\right]
\left[
\begin{array}{c}
     u_1  \\
     u_2  \\
     u_3  \\
     u_4  \\
     u_5  \\
     u_6
\end{array}
\right].
$$
\endgroup
{In the upper part of Table \ref{table:CECA_example}, we give the update process of $I_i^\ro,J_i^\ro$ in the 3 rounds when the initial $u_1$--$u_6$ are assigned to be values $1$--$6$.}

\begin{table}[]
\label{table:CECA_example}
\centering
\begin{tabular}{c|ll|ll|ll|ll|ll|ll}
\hline
\multicolumn{13}{c}{2-port CECA}\\\hline
round
$r$& $I^\ro_1$& $J^\ro_1$& $I^\ro_2$& $J^\ro_2$& $I^\ro_3$& $J^\ro_3$& $I^\ro_4$& $J^\ro_4$& $I^\ro_5$& $J^\ro_5$& $I^\ro_6$& $J^\ro_6$\\\hline\hline
0  & 1    & 0    & 2    & 0    & 3    & 0    & 4    & 0    & 5    & 0    & 6    & 0    \\\hline
   & \multicolumn{12}{c}{$n_0 = 0$, $\delta_0 = 1$}                                    \\\hline
1  & 3.5  & 6    & 1.5  & 1    & 2.5  & 2    & 3.5  & 3    & 4.5  & 4    & 5.5  & 5    \\\hline
   & \multicolumn{12}{c}{$n_1 = 1$, $\delta_1 = 0$}                                    \\\hline
2  & 4    & 5.5  & 3    & 3.5  & 2    & 1.5  & 3    & 2.5  & 4    & 3.5  & 5    & 4.5  \\\hline
   & \multicolumn{12}{c}{$n_2 = 2$, $\delta_2 = 1$}                                    \\\hline
3  & 3.5  & 4    & 3.5  & 3.8  & 3.5  & 3.6  & 3.5  & 3.4  & 3.5  & 3.2  & 3.5  & 3    \\\hline\hline
\end{tabular}
\medskip

\begin{tabular}{c|ll|ll|ll|ll|ll|ll}
\hline
\multicolumn{13}{c}{New algorithm~(1-port allreduce averaging)}\\\hline
round
$r$& $I^\ro_1$& $J^\ro_1$& $I^\ro_2$& $J^\ro_2$& $I^\ro_3$& $J^\ro_3$& $I^\ro_4$& $J^\ro_4$& $I^\ro_5$& $J^\ro_5$& $I^\ro_6$& $J^\ro_6$   \\\hline\hline
0  & 1    & 0    & 2    & 0    & 3    & 0    & 4    & 0    & 5    & 0    & 6    & 0       \\\hline
   & \multicolumn{12}{c}{$n_0 = 1$, $\delta_0 = 1$, agents $(1~2)$ $(3~4)$ $(5~6)$ exchange $I$} \\\hline
1  & 1.5  & 2    & 1.5  & 1    & 3.5  & 4    & 3.5  & 3    & 5.5  & 6    & 5.5  & 5       \\\hline
   & \multicolumn{12}{c}{$n_1 = 1$, $\delta_1 = 0$, agents $(1~4)$ $(2~5)$ $(3~6)$ exchange $J$} \\\hline
2  & 2    & 2.5  & 3    & 3.5  & 4    & 4.5  & 3    & 2.5  & 4    & 3.5  & 5    & 4.5     \\\hline
   & \multicolumn{12}{c}{$n_2 = 2$, $\delta_2 = 1$, agents $(1~6)$ $(2~3)$ $(4~5)$ exchange $I$} \\\hline
3  & 3.5  & 4    & 3.5  & 3.8  & 3.5  & 3.6  & 3.5  & 3.4  & 3.5  & 3.2  & 3.5  & 3       \\\hline
\end{tabular}
\caption{Illustration of the two allreduce averaging algorithms applied to 6 agents with agents 1--6 having initial numbers 1--6. Both algorithms take $\lceil\log_2(6)\rceil=3$ rounds to achieve $I^{(r)}_i\equiv 3.5$ for all $i=1,\dots, 6$.}
\end{table}

\subsubsection{Optimal allreduce algorithm for the 1-port message passing system}
\label{append:CECA-2P}
Here, we give a more detailed description of the optimal allreduce algorithm introduced in Section \ref{sec:1p-opt-allreduce-ave} 

{We calculate that $\tau=\lceil \log_2 n\rceil$ rounds are needed to reach consensus.} We compute the necessary parameters $\delta_r, n_r, r=0,1,2,\ldots,\tau-1$ as in \eqref{eq:bin-n-1} and \eqref{eq:iter-nr}.

Before the communication starts, we initialize $I^{(0)}_i=u_i, J^{(0)}_i=0$. 
In the $r^{\text{th}}$ round, $r=0,1,2,\ldots,\tau-1$, if $i$ odd, we pair up agents $i$ and $i+2n_r+1$; if $i$ even, we pair up agents $i$ and $i-2n_r-1$. Each agent exchanges information with its peer. %
If $\delta_r=1$, the peers exchange $I^\ro$; if $\delta_r=0$, they exchange $J^\ro$ instead.

{Let $\omega^\ro_i$ denote the index of the agent who sends message to agent $i$ in the $r^{\textup{th}}$ round.} If $\delta_r=1$,
\begin{align*}
I^{(r+1)}_i&=\frac{1}{2}I^\ro_i+\frac{1}{2}I^\ro_{\omega^\ro_i}\\
J^{(r+1)}_i&=\frac{n_r}{2n_r+1}J^\ro_i+\frac{n_r+1}{2n_r+1}I^\ro_{\omega^\ro_i}.
\end{align*}
If $\delta_r=0$,
\begin{align*}
I^{(r+1)}_i&=\frac{n_r+1}{2n_r+1}I^\ro_i+\frac{n_r}{2n_r+1}J^\ro_{\omega^\ro_i},\\
J^{(r+1)}_i&=\frac{1}{2}J^\ro_i+\frac{1}{2}J^\ro_{\omega^\ro_i}.
\end{align*}
After the above averaging process, each odd agent $i$ has information
$$I_i^{(r+1)}=\frac{1}{n_{r+1}+1}\sum_{j=0}^{n_{r+1}}u_{i+j},\quad 
J_i^{(r+1)}=\frac{1}{n_{r+1}}\sum_{j=1}^{n_{r+1}}u_{i+j}.$$
Each even agent $i$ has information
$$
I_i^{(r+1)}=\frac{1}{n_{r+1}+1}\sum_{j=0}^{n_{r+1}}u_{i-j},\quad 
J_i^{(r+1)}=\frac{1}{n_{r+1}}\sum_{j=1}^{n_{r+1}}u_{i-j}.$$
After $\tau$ rounds, each agent has the consensus information $I^{(\tau)}_i=\frac{1}{n}\sum_{j=1}^n u_j$. 

\textbf{Example: }We give an example when $n=6$. The calculation of $\delta_r$ and $n_r$ is the same as in the example of Section \ref{sec:CECA-allreduce-alg}. Before the communications start, each agent $i$ has information $I^{(0)}_i=u_i, J^{(0)}_i=0$. In the $0^{\text{th}}$ round, {we pair up agents as follows: 1 with 2, 3 with 4, and 5 with 6, in a peer-to-peer manner.} Each agent sends its $I^{(0)}$ information to its peer  {and receives the same type of $I^{(0)}$ information in return. After averaging, $I_i,J_i$ are updated as}
\begingroup
\renewcommand{\arraystretch}{1.2} %
$$
\vI^{(1)}
=
\left[
\begin{array}{cccccc}
    \frac{1}{2} & \frac{1}{2} & 0 & 0 & 0 & 0 \\
    \frac{1}{2} & \frac{1}{2} & 0 & 0 & 0 & 0 \\
    0 & 0 & \frac{1}{2} & \frac{1}{2} & 0 & 0 \\
    0 & 0 & \frac{1}{2} & \frac{1}{2} & 0 & 0 \\
    0 & 0 & 0 & 0 & \frac{1}{2} & \frac{1}{2} \\
    0 & 0 & 0 & 0 & \frac{1}{2} & \frac{1}{2}
\end{array}
\right]
\left[
\begin{array}{c}
     u_1  \\
     u_2  \\
     u_3  \\
     u_4  \\
     u_5  \\
     u_6
\end{array}
\right],\quad
\vJ^{(1)}
=
\left[
\begin{array}{cccccc}
    0 & 1 & 0 & 0 & 0 & 0 \\
    1 & 0 & 0 & 0 & 0 & 0 \\
    0 & 0 & 0 & 1 & 0 & 0 \\
    0 & 0 & 1 & 0 & 0 & 0 \\
    0 & 0 & 0 & 0 & 0 & 1 \\
    0 & 0 & 0 & 0 & 1 & 0 
\end{array}
\right]
\left[
\begin{array}{c}
     u_1  \\
     u_2  \\
     u_3  \\
     u_4  \\
     u_5  \\
     u_6
\end{array}
\right].
$$
In the $1^{\textup{st}}$ round, $\delta_1=0$. We pair up agents $(1,4),(2,5),(3,6)$. {Each agent sends $J^{(1)}$ information to its peer and also receives $J^{(1)}$ information in return.} After averaging, we get
$$
\vI^{(2)}
=
\left[
\begin{array}{cccccc}
    \frac{1}{3} & \frac{1}{3} & \frac{1}{3} & 0 & 0 & 0 \\
    \frac{1}{3} & \frac{1}{3} & 0 & 0 & 0 & \frac{1}{3} \\
    0 & 0 & \frac{1}{3} & \frac{1}{3} & \frac{1}{3} & 0 \\
    0 & \frac{1}{3} & \frac{1}{3} & \frac{1}{3} & 0 & 0 \\
    \frac{1}{3} & 0 &  0 & 0 & \frac{1}{3} & \frac{1}{3}  \\
    0 & 0 & 0 & \frac{1}{3} & \frac{1}{3} & \frac{1}{3} 
\end{array}
\right]
\left[
\begin{array}{c}
     u_1  \\
     u_2  \\
     u_3  \\
     u_4  \\
     u_5  \\
     u_6
\end{array}
\right],\quad
\vJ^{(2)}
=
\left[
\begin{array}{cccccc}
    0 & \frac{1}{2} & \frac{1}{2}& 0 & 0 & 0  \\
    \frac{1}{2} & 0 & 0 & 0 & 0 & \frac{1}{2} \\
    0 & 0 & 0 & \frac{1}{2} & \frac{1}{2} & 0 \\
    0 & \frac{1}{2} & \frac{1}{2} & 0 & 0 & 0 \\
    \frac{1}{2} & 0 & 0 & 0  & 0 & \frac{1}{2}  \\
    0 & 0 & 0 & \frac{1}{2} & \frac{1}{2} & 0 
\end{array}
\right]
\left[
\begin{array}{c}
     u_1  \\
     u_2  \\
     u_3  \\
     u_4  \\
     u_5  \\
     u_6
\end{array}
\right].
$$
In the $2^{\textup{nd}}$ round, $\delta_2=1$. The pairing mode for agents is $(1,6),(2,3),(4,5)$. {Each agent exchanges $J^{(2)}$ information with its peer.} After merging the received information into the previous information, we have
$$
\vI^{(3)}
=
\left[
\begin{array}{cccccc}
    \frac{1}{6} & \frac{1}{6} & \frac{1}{6} & \frac{1}{6} & \frac{1}{6} & \frac{1}{6} \\
    \frac{1}{6} & \frac{1}{6} & \frac{1}{6} & \frac{1}{6} & \frac{1}{6} & \frac{1}{6} \\
    \frac{1}{6} & \frac{1}{6} & \frac{1}{6} & \frac{1}{6} & \frac{1}{6} & \frac{1}{6} \\
    \frac{1}{6} & \frac{1}{6} & \frac{1}{6} & \frac{1}{6} & \frac{1}{6} & \frac{1}{6} \\
    \frac{1}{6} & \frac{1}{6} & \frac{1}{6} & \frac{1}{6} & \frac{1}{6} & \frac{1}{6} \\
    \frac{1}{6} & \frac{1}{6} & \frac{1}{6} & \frac{1}{6} & \frac{1}{6} & \frac{1}{6}
\end{array}
\right]
\left[
\begin{array}{c}
     u_1  \\
     u_2  \\
     u_3  \\
     u_4  \\
     u_5  \\
     u_6
\end{array}
\right],\quad
\vJ^{(3)}
=
\left[
\begin{array}{cccccc}
    0 & \frac{1}{5} & \frac{1}{5} & \frac{1}{5} & \frac{1}{5} & \frac{1}{5} \\
    \frac{1}{5} & 0 & \frac{1}{5} & \frac{1}{5} & \frac{1}{5} & \frac{1}{5} \\
    \frac{1}{5} & \frac{1}{5} & 0 & \frac{1}{5} & \frac{1}{5} & \frac{1}{5} \\
    \frac{1}{5} & \frac{1}{5} & \frac{1}{5} & 0 & \frac{1}{5} & \frac{1}{5} \\
    \frac{1}{5} & \frac{1}{5} & \frac{1}{5} & \frac{1}{5} & 0 & \frac{1}{5} \\
    \frac{1}{5} & \frac{1}{5} & \frac{1}{5} & \frac{1}{5} & \frac{1}{5} & 0
\end{array}
\right]
\left[
\begin{array}{c}
     u_1  \\
     u_2  \\
     u_3  \\
     u_4  \\
     u_5  \\
     u_6
\end{array}
\right].
$$
\endgroup
{In the lower part of Table \ref{table:CECA_example}, we give the changes of $I^\ro_i,J^\ro_i$ when the initial values $u_1$--$u_6$ are set to be $1$--$6$.}

\subsection{Transfer matrix family}
\label{append:trans-matrix-family}
In this section, we prove the properties of transfer matrix family $\mathscr{F}$ defined in \eqref{eq:F-family}.
The transfer matrices $\vW, \vW_g$ defined in \eqref{eq:W-Wg-d1} and \eqref{eq:W-Wg-d0} are in this family.
We show the useful properties of the matrix family in the following lemmas.
\begin{lemma}
\label{lem:F-semigroup}
The matrix family $\mathscr{F}$ is closed under matrix multiplication. If $\vW,\vV\in\mathscr{F}$, then $\vW\vV\in\mathscr{F}$. Thus the family $\mathscr{F}$ is a semigroup under matrix multiplication.
\end{lemma}
\begin{proof}
Let
$$
\vW=\left[
\begin{array}{cc}
    c_w \vW_{1,1} & (1-c_w)\vW_{1,2} \\
    d_w\vW_{2,1} & (1-d_w)\vW_{2,2} 
\end{array}
\right],\quad
\vV=\left[
\begin{array}{cc}
    c_v \vV_{1,1} & (1-c_v)\vV_{1,2} \\
    d_v\vV_{2,1} & (1-d_v)\vV_{2,2} 
\end{array}
\right].
$$
Here $0\leq c_w,c_v,d_w,d_v \leq 1$, $\vW_{i,j}, \vV_{i,j}$ are doubly stochastic, $i,j\in\{1,2\}$.
$$
\vW\vV=
\left[
\begin{array}{cc}
c_wc_v\vW_{1,1}\vV_{1,1}+(1-c_w)d_v\vW_{1,2}\vV_{2,1} &
c_w(1-c_v)\vW_{1,1}\vV_{1,2}+(1-c_w)(1-d_v)\vW_{1,2}\vV_{2,2}\\
d_wc_v\vW_{2,1}\vV_{1,1}+(1-d_w)d_v\vW_{2,2}\vV_{2,1} &
d_w(1-c_v)\vW_{2,1}\vV_{1,2}+(1-d_w)(1-d_v)\vW_{2,2}\vV_{2,2}
\end{array}
\right]
$$
Consider the first block, if $c_wc_v+(1-c_w)d_w=0$,
$$
c_wc_v\vW_{1,1}\vV_{1,1}+(1-c_w)d_v\vW_{1,2}\vV_{2,1}=\vzero_n=(c_wc_v+(1-c_w)d_w)\vI_n
$$
If $c_wc_v+(1-c_w)d_w>0$,
\begin{multline*}
c_wc_v\vW_{1,1}\vV_{1,1}+(1-c_w)d_v\vW_{1,2}\vV_{2,1}=\\
(c_wc_v+(1-c_w)d_v)\left(
\frac{c_wc_v}{c_wc_v+(1-c_w)d_v}\vW_{1,1}\vV_{1,1}+\frac{(1-c_w)d_v}{c_wc_v+(1-c_w)d_v}\vW_{1,2}\vV_{2,1}
\right).
\end{multline*}
The product of two doubly stochastic matrix is doubly stochastic. The convex combination of doubly stochastic matrices is still doubly stochastic. So, we can always find a doubly stochastic matrix $\vU_{1,1}$ such that 
$$
c_wc_v\vW_{1,1}\vV_{1,1}+(1-c_w)d_v\vW_{1,2}\vV_{2,1}=(c_wc_v+(1-c_w)d_w)\vU_{1,1}.
$$
Similarly, we can find doubly stochastic matrices $\vU_{i,j}$, $i,j\in\{1,2\}$, such that
$$
\vW\vV=
\left[
\begin{array}{cc}
(c_wc_v+(1-c_w)d_v)\vU_{1,1} &
(c_w(1-c_v)+(1-c_w)(1-d_v))\vU_{1,2}\\
(d_wc_v+(1-d_w)d_v)\vU_{2,1} &
(d_w(1-c_v)+(1-d_w)(1-d_v))\vU_{2,2}
\end{array}
\right]\in \mathscr{F}.
$$
\end{proof}

\begin{lemma}
\label{lem:F-matrix-upper-bound}
For any $\vW\in\mathscr{F}$, $\|\vW\|\leq \sqrt 2$.
\end{lemma}
\begin{proof}
Let
$$\vW=\left[
\begin{array}{cc}
    c \vW_{1,1} & (1-c)\vW_{1,2} \\
    d\vW_{2,1} & (1-d)\vW_{2,2} 
\end{array}
\right].$$
{Here $0\leq c,d \leq 1$, $\vW_{i,j}, \vV_{i,j}$ are doubly stochastic, $i,j\in\{1,2\}$.}
For any $\vx,\vy\in\bbR^n$,
\begin{align*}
\left\|\vW\left[
\begin{array}{c}
     \vx  \\
     \vy
\end{array}
\right]\right\|^2&=\left\|
\left[
\begin{array}{c}
    c \vW_{1,1}\vx+(1-c)\vW_{1,2}\vy \\
    d\vW_{2,1}\vx+ (1-d)\vW_{2,2}\vy 
\end{array}
\right]
\right\|^2\\
&=\|c \vW_{1,1}\vx+(1-c)\vW_{1,2}\vy\|^2+\|d\vW_{2,1}\vx+ (1-d)\vW_{2,2}\vy \|^2\\
&\stackrel{(a)}{\leq} c \|\vW_{1,1}\vx \|^2+(1-c)\|\vW_{1,2}\vy\|^2+d\|\vW_{2,1}\vx\|^2+ (1-d)\|\vW_{2,2}\vy \|^2\\
&\stackrel{(b)}{\leq}c\|\vx\|^2+(1-c)\|\vy\|^2+d\|\vx\|^2+ (1-d)\|\vy \|^2\\
&\leq 2(\|\vx\|^2+\|\vy\|^2).
\end{align*}
In the above inequalities, $(a)$ holds because of Jensen's inequality; $(b)$ holds because the norm of a doubly stochastic matrix is bounded by 1.
\end{proof}
\begin{lemma}
\label{lem:commutativity}
For any $\vW\in\mathscr{F}$,
\begin{equation*}
\vW
\left[
\begin{array}{cc}
    \frac{1}{n}\bbone_n\bbone_n^\top & \vzero_n \\
    \vzero_n & \frac{1}{n}\bbone_n\bbone_n^\top
\end{array}
\right]
=
\left[
\begin{array}{cc}
    \frac{1}{n}\bbone_n\bbone_n^\top & \vzero_n \\
    \vzero_n & \frac{1}{n}\bbone_n\bbone_n^\top
\end{array}
\right]
\vW.
\end{equation*}
\end{lemma}
\begin{proof}
Let
$$\vW=\left[
\begin{array}{cc}
    c \vW_{1,1} & (1-c)\vW_{1,2} \\
    d\vW_{2,1} & (1-d)\vW_{2,2} 
\end{array}
\right].$$
{Here $0\leq c,d \leq 1$, $\vW_{i,j}, \vV_{i,j}$ are doubly stochastic, $i,j\in\{1,2\}$.}
We have
\begin{align*}
\vW
\left[
\begin{array}{cc}
    \frac{1}{n}\bbone_n\bbone_n^\top & \vzero_n \\
    \vzero_n & \frac{1}{n}\bbone_n\bbone_n^\top
\end{array}
\right]
&=
\left[
\begin{array}{cc}
    c \vW_{1,1}\frac{1}{n}\bbone_n\bbone_n^\top & (1-c)\vW_{1,2}\frac{1}{n}\bbone_n\bbone_n^\top \\
    d\vW_{2,1}\frac{1}{n}\bbone_n\bbone_n^\top & (1-d)\vW_{2,2}\frac{1}{n}\bbone_n\bbone_n^\top 
\end{array}
\right]\\
&=\left[
\begin{array}{cc}
    c\frac{1}{n}\bbone_n\bbone_n^\top & (1-c)\frac{1}{n}\bbone_n\bbone_n^\top \\
    d\frac{1}{n}\bbone_n\bbone_n^\top & (1-d)\frac{1}{n}\bbone_n\bbone_n^\top
\end{array}
\right]\\
&=
\left[
\begin{array}{cc}
    c \frac{1}{n}\bbone_n\bbone_n^\top\vW_{1,1} & (1-c)\frac{1}{n}\bbone_n\bbone_n^\top\vW_{1,2} \\
    d\frac{1}{n}\bbone_n\bbone_n^\top\vW_{2,1} & (1-d)\frac{1}{n}\bbone_n\bbone_n^\top\vW_{2,2} 
\end{array}
\right]\\
&=\left[
\begin{array}{cc}
    \frac{1}{n}\bbone_n\bbone_n^\top & \vzero_n \\
    \vzero_n & \frac{1}{n}\bbone_n\bbone_n^\top
\end{array}
\right]
\vW.
\end{align*}
\end{proof}

\begin{lemma}[matrix consensus]
In the 2-port system. When the communication matrix $\vP^\ko$ is sampled from \eqref{eq:2p-gossip-matrix-d1} or \eqref{eq:2p-gossip-matrix-d0} {according to $\delta_r$}, and the mixing matrix $\vW^\ko$ is sampled from \eqref{eq:W-Wg-d1} or \eqref{eq:W-Wg-d0} {according to $\delta_r$}. The product of  matrices $\vW^\ko$ satisfies
\begin{gather}
\label{eq:matrix-consensus1}
\prod_{k=0}^{\tau-1}\vW^\ko=
\left[
\begin{array}{cc}
    \frac{1}{n}\bbone_n\bbone_n^\top & \vzero_n \\
    \frac{1}{n-1}(\bbone_n\bbone_n^\top-\vI_n) & \vzero_n
\end{array}\right],\\
\label{eq:matrix-consensus2}
\prod_{k=0}^{t}\vW^\ko=
\left[
\begin{array}{cc}
    \frac{1}{n}\bbone_n\bbone_n^\top & \vzero_n \\
    \frac{1}{n}\bbone_n\bbone_n^\top & \vzero_n
\end{array}\right], \quad \forall t\geq \tau.
\end{gather}
\end{lemma}
\begin{proof}
Since $\delta_0=1, n_0=0$, we have
$$
\vW^{(0)}=
\left[
\begin{array}{cc}
    \frac{1}{2}\vI_n+\frac{1}{2}\vP^{(0)} & \vzero_n \\
    \vP^{(0)} & \vzero_n
\end{array}
\right].
$$
So, for any $s\geq 0$, $\prod_{k=0}^s \vW^\ko$ is of the form 
$$
\left[
\begin{array}{cc}
    * & \vzero_n \\
    * & \vzero_n
\end{array}
\right].
$$
Here $*$ stands for an undetermined $n\times n$ matrix block. By induction, for $s\leq \tau$,
\begin{gather*}
\prod_{k=0}^s\vW^\ko =
\left[
\begin{array}{cc}
    \vA & \vzero_n \\
    \vB & \vzero_n
\end{array}
\right],\\
\vA_{i,j}=\left\{
\begin{array}{cc}
    \frac{1}{n_{s+1}+1}, & (i-j)\in\{0,1,2,\ldots,n_{s+1}\}\Mod{n} \\
    0, & \textup{otherwise}
\end{array}
\right.,\quad 
\vB_{i,j}=\left\{
\begin{array}{cc}
    \frac{1}{n_{s+1}}, & (i-j)\in\{1,2,\ldots,n_{s+1}\}\Mod{n}  \\
    0, & \textup{otherwise}
\end{array}
\right..
\end{gather*}
When $s=\tau-1$, we get \eqref{eq:matrix-consensus1}. For $t\geq \tau$,
\begin{align*}
\prod_{k=0}^t\vW^\ko&=
\left(\prod_{k=\tau+1}^{t}\vW^\ko\right)\cdot
\vW^{(\tau)}\cdot
\left(\prod_{k=0}^{\tau-1}\vW^\ko\right)\\
&=\left(\prod_{k=\tau+1}^{t}\vW^\ko\right)\cdot
\left[
\begin{array}{cc}
    \frac{1}{2}\vI_n+\frac{1}{2}\vP^{(0)} & \vzero_n \\
    \vP^{(0)} & \vzero_n
\end{array}
\right]\cdot
\left[
\begin{array}{cc}
    \frac{1}{n}\bbone_n\bbone_n^\top & \vzero_n \\
    \frac{1}{n-1}(\bbone_n\bbone_n^\top-\vI_n) & \vzero_n
\end{array}\right]\\
&=\left(\prod_{k=\tau+1}^{t}\vW^\ko\right)\cdot
\left[
\begin{array}{cc}
    \frac{1}{n}\bbone_n\bbone_n^\top & \vzero_n \\
    \frac{1}{n}\bbone_n\bbone_n^\top & \vzero_n
\end{array}\right]\\
&\stackrel{(a)}=\left[
\begin{array}{cc}
    \frac{1}{n}\bbone_n\bbone_n^\top & \vzero_n \\
    \frac{1}{n}\bbone_n\bbone_n^\top & \vzero_n
\end{array}\right].
\end{align*}
Here, $\left(\prod_{k=\tau+1}^{t}\vW^\ko\right)\in\mathscr{F}$ because $\mathscr{F}$ is a semigroup according to Lemma \ref{lem:F-semigroup}, any matrix in $\mathscr{F}$ is row stochastic, so (a) holds. We get \eqref{eq:matrix-consensus2}.
\end{proof}
\subsection{Convex analysis tools}
{In this section, we list some convex analysis concepts and inequalities useful in the algorithm convergence analysis.
\begin{definition}[$L$-smoothness]
A differentiable function $f:\bbR^d\rightarrow\bbR$ is called $L$-smooth if for all $\bvx,\bvy\in\bbR^d$, we have
\begin{equation}
\label{eq:L-smooth-eq-def1}
\|\nabla f(\bvx)-\nabla f(\bvy)\|\leq L\|\bvx-\bvy\|.
\end{equation}
\end{definition}
For an $L$-smooth function $f$, we have the following inequality,
\begin{equation}
\label{eq:L-smooth-eq-def2}
f(\bvx)\leq f(\bvy)+\langle \nabla f(\bvy),\bvx-\bvy \rangle+\frac{L}{2}\|\bvx-\bvy\|^2,\quad \forall \bvx,\bvy\in\bbR^d.
\end{equation}}

{
\begin{definition}[Convexity]
We call a function $f:\bbR^d\rightarrow\bbR$ convex if for all $\bvx,\bvy\in\bbR^d$, we have
\begin{equation}
\label{eq:conv-def}
f(\bvx)\geq f(\bvy)+\langle \nabla f(\bvy),\bvx-\bvy \rangle.
\end{equation}
\end{definition}}

{Given a function $f$, let} $\vx^*$ be the minimizer of $f$. If $f$ is both $L$-smooth and convex, we have
\begin{equation}
\label{eq:conv-Lip-ineq}
\|\nabla f(\vx)\|^2\leq 2L(f(\vx)-f(\vx^*)),\quad \forall \vx\in\bbR^d.
\end{equation}

\subsection{Convergence property of DSGD-CECA}
\label{append:conv_thm}
{In this section, we provide the convergence proof for DSGD-CECA-2P, as presented in Algorithm \ref{alg:dsgd-ceca}. DSGD-CECA-1P differs from DSGD-CECA-2P only in the sampling of the communication matrix $\vP^\ko$. Their convergence proofs follow a similar approach. For simplicity, we use DSGD-CECA to refer to DSGD-CECA-2P in this section. To establish the convergence property of DSGD-CECA, we first present several supporting lemmas.}
\begin{lemma}[\sc relation between averaged variables]
\label{lem:bar-eq-iter}
Consider DSGD-CECA-2P recursions in Algorithm \ref{alg:dsgd-ceca}. It holds for any $k=0,1,2,\cdots$ that 
$
\bar\vx^\ko=\bar\vy^\ko=\bar\vz^\ko$ and
$$
\bar \vx^\kp=\bar \vx^\ko-\gamma \bar\ve^\ko,
$$
where $\bar\ve^\ko = \frac{1}{n} \nabla F(\vz^\ko; \vxi^\ko)^\top \mathds{1}_n$. 
\end{lemma}
\begin{proof}
Left-multiplying $\frac{1}{n}\bbone_n^\top$ to both sides of \eqref{eq:x-iter-dsgd}, \eqref{eq:y-iter-dsgd}, we get
\begin{align}
\label{eq:bar-x-iter-dsgd}
\bar\vx^\kp&=a^\ko(\bar\vx^\ko-\gamma \bar\ve^\ko)+(1-a^\ko)(\bar\vz^\ko-\gamma\bar\ve^\ko),\\
\label{eq:bar-y-iter-dsgd}
\bar\vy^\kp&=b^\ko(\bar\vy^\ko-\gamma\bar\ve^\ko)+(1-b^\ko)(\bar\vz^\ko-\gamma\bar\ve^\ko).
\end{align}
The difference between \eqref{eq:bar-x-iter-dsgd} and \eqref{eq:bar-y-iter-dsgd} is
$$
\bar\vx^\kp-\bar\vy^\kp=\theta^\ko(\bar\vx^\ko-\bar\vy^\ko),
$$
where $\theta^\ko=b^\ko$ if $\bar\vz^\ko=\bar\vx^\ko$ when $\delta_r = 1$, and $\theta^\ko=a^\ko$ if $\bar\vz^\ko=\bar\vy^\ko$ when $\delta_r = 0$. Since $\vx^{(0)}=\vy^{(0)}$ in the initial setting, we have $\bar\vx^{(0)}=\bar\vy^{(0)}$. Inductively, we have $\bar\vx^\ko=\bar\vy^\ko=\bar\vz^\ko$, for any $k\geq0$. Putting this result to \eqref{eq:bar-x-iter-dsgd}, we have
$$
\bar\vx^\kp=a^\ko\bar\vx^\ko+(1-a^\ko)\bar\vz^\ko-\gamma\bar\ve^\ko=\bar\vx^\ko-\gamma\bar\ve^\ko.
$$
\end{proof}

\begin{lemma}[\sc Descent lemma] \label{lem:descent}
Under Assumptions \ref{ass:Lip}--\ref{ass:grad-noise} and step-size $\gamma < \frac{1}{4L}$, it holds for $k=0,1,\cdots$ that 
\begin{align}\label{eq-descent-lamma}
\mathbb{E} f(\bar{\vx}^{(k+1)})\le &\ \mathbb{E} f(\bar{\vx}^{(k)}) - \frac{\gamma}{4}\mathbb{E}\|\nabla f(\bar{\vz}^{(k)})\|^2 + \frac{\gamma^2 L \sigma^2}{2n} \nonumber \\
&\ + \frac{3\gamma L^2}{4n} \Big(\mathbb{E}\|\vx^\ko-\bbone_n (\bar\vx^\ko)^\top\|^2_F+ \mathbb{E}\|\vy^\ko-\bbone_n (\bar\vy^\ko)^\top\|^2_F\Big).
\end{align}
\end{lemma}
\begin{proof}
We first introduce the following filtration to simplify the analysis
\begin{equation}
\label{eq:hist-sigma-alg}
\mathcal{F}^{(k)}=\sigma\{\vx^{(0)},\vxi^{(0)},\vxi^{(1)},\vxi^{(2)},\ldots,\vxi^{(k-1)}\}.
\end{equation}
It is the $\sigma$ algebra of all random variables before the $k^{\text{th}}$ iteration.
With $\bar \vx^\kp=\bar \vx^\ko-\gamma \bar\ve^\ko$ and Assumption \ref{ass:Lip}, it holds that
\begin{align}\label{xcn}
\bbE[f(\bar\vx^\kp)|\mathcal{F}^{(k-1)}] \le&\ f(\bar\vx^\ko) - \gamma \bbE[\langle \nabla f(\bar\vx^\ko), \bar\ve^\ko \rangle |\mathcal{F}^{(k-1)}] + \frac{L\gamma^2}{2}\bbE[\|\bar\ve^\ko\|^2|\mathcal{F}^{(k-1)}] \nonumber \\
=&\ f(\bar\vx^\ko) - \gamma \langle \nabla f(\bar\vx^\ko), \frac{1}{n}\sumin \nabla f_i(\bvz_i^\ko) \rangle + \frac{L\gamma^2}{2}\|\frac{1}{n}\sumin \nabla f_i(\bvz_i^\ko)\|^2 + \frac{\gamma^2L\sigma^2}{2n}  
\end{align}
Note that 
\begin{align}\label{237}
	-\langle \nabla f(\bar{\vx}^{{(k)}}),  \frac{\gamma}{n}\sum_{i=1}^n \nabla f_i(\bvz_i^{(k)})\rangle &= -\langle \nabla f(\bar{\vz}^{{(k)}}),  \frac{\gamma}{n}\sum_{i=1}^n [\nabla f_i(\bvz_i^{(k)}) - \nabla f_i(\bar{\vz}^{(k)}) + \nabla f_i(\bar{\vz}^{(k)})]\rangle \nonumber \\
	&\le -\gamma \|\nabla f(\bar{\vz}^{{(k)}})\|^2 + \frac{\gamma}{2}\|\nabla f(\bar{\vz}^{{(k)}})\|^2 + \frac{\gamma}{2n}\sum_{i=1}^n\|\nabla f_i(\bvz_i^{(k)}) - \nabla f_i(\bar{\vz}^{(k)})\|^2 \nonumber \\
	&\le -\frac{\gamma}{2}\|\nabla f(\bar{\vz}^{{(k)}})\|^2 + \frac{\gamma L^2}{2n}\|\vz^\ko-\bbone_n (\bar\vz^\ko)^\top\|_F^2,
\end{align}
where the first equality holds because $\bar\vz^\ko = \bar\vx^\ko$ for any $k=0,1,2,\cdots$. Furthermore, it also holds that 
\begin{align}\label{2nd}
	\|\frac{1}{n}\sum_{i=1}^n \nabla f_i(\bvz_i^{(k)})\|^2 \le \frac{2L^2}{n}\|\vz^\ko-\bbone_n (\bar\vz^\ko)^\top\|_F^2 + 2\|\nabla f(\bar{\vz}^{(k)})\|^2
\end{align}
and 
\begin{align}\label{zxy}
\|\vz^\ko-\bbone_n (\bar\vz^\ko)^\top\|^2_F\leq \|\vx^\ko-\bbone_n (\bar\vx^\ko)^\top\|^2_F+\|\vy^\ko-\bbone_n (\bar\vy^\ko)^\top\|^2_F
\end{align}
Substituting \eqref{237}, \eqref{2nd} and \eqref{zxy} into \eqref{xcn}, taking expectations over $\mathcal{F}^{(k)}$, and using the fact that $\gamma < \frac{1}{4L}$, we reach \eqref{eq-descent-lamma}.
\end{proof}
To bound the term $\bbE\|\vx^\ko-\bbone_n (\bar\vx^\ko)^\top\|^2_F+\bbE\|\vy^\ko-\bbone_n (\bar\vy^\ko)^\top\|^2_F$ in \eqref{eq-descent-lamma}, we introduce the auxiliary variables
$$
\vphi^\ko=
\left[
\begin{array}{c}
     \vx^\ko-\bbone_n (\bar\vx^\ko)^\top \\
     \vy^\ko-\bbone_n (\bar\vy^\ko)^\top
\end{array}
\right]\in \mathbb{R}^{2n\times d}, \quad 
\vpsi^\ko=
\left[
\begin{array}{c}
     \vg^\ko-\bbone_n (\bar\vg^\ko)^\top \\
     \vh^\ko-\bbone_n (\bar\vh^\ko)^\top
\end{array}
\right]\in \mathbb{R}^{2n\times d}.
$$
It holds that $\|\vphi^\ko\|^2_F=\|\vx^\ko-\bbone_n (\bar\vx^\ko)^\top\|^2_F+\|\vy^\ko-\bbone_n (\bar\vy^\ko)^\top\|^2_F$.

\begin{lemma}[\sc consensus lemma] 
\label{lem:consenus-dsgd}
Under Assumptions \ref{ass:Lip}-\ref{ass:data-hetero}, if the learning rate $\gamma \le \frac{1}{8\tau L}$, it holds that 
\begin{align}\label{eq-consensus}
\avekT \bbE\|\vphi^\ko\|^2_F
\leq \frac{2}{T+1}\sum_{k=0}^\tau\bbE\|\vphi^\ko\|^2_F + 32 n\tau^2\gamma^2(\sigma^2 + 2b^2) 
\end{align}
\end{lemma}
\begin{proof}
We prove \eqref{eq-consensus} in three steps.

\textbf{Step I.} In this step, we will provide a rough upper bound to $\avekT \bbE\|\vphi^\ko\|^2_F$. By left-multiplying both sides of \eqref{eq:iter-with-trans-matrix} by 
$$\vI_{2n}-
\left[
\begin{array}{cc}
    \frac{1}{n}\bbone_n\bbone_n^\top & \vzero_n \\
    \vzero_n & \frac{1}{n}\bbone_n\bbone_n^\top
\end{array}
\right]$$ 
and utilizing the commutativity property proved in Lemma \ref{lem:commutativity}, we have
\begin{equation}
\label{eq:iter-trans-matrx-phi-psi}
\vphi^\kp=\vW^\ko\vphi^\ko-\gamma\vW^\ko_g \vpsi^\ko.
\end{equation}
For any $k\ge \tau$, we let $m = \lfloor k/\tau \rfloor - 1 $ and hence $k - m\tau \ge \tau$. Keep iterating \eqref{eq:iter-trans-matrx-phi-psi}, we have for any $k\geq \tau$ that 
\begin{align*}
\vphi^\ko&=\vW^{(k-1)}\vphi^{(k-1)}-\gamma\vW^{(k-1)}_g \vpsi^{(k-1)}\\
&\stackrel{(a)}{=}\left(\prod_{j=m\tau}^{k-1}\vW^{(j)}\right)\vphi^{(m\tau)}-\gamma\sum_{l=1}^{k-m\tau}\left(\prod_{j=k-l+1}^{k-1}\vW^{(j)}\right)\vW^{(k-l)}_g\vpsi^{(k-l)}\\
&\stackrel{(b)}{=}-\gamma\sum_{l=1}^{k-m\tau}\left(\prod_{j=k-l+1}^{k-1}\vW^{(j)}\right)\vW^{(k-l)}_g\vpsi^{(k-l)}. 
\end{align*}
where in equality (a) we define $\prod_{j=p}^{q} \vW^{(j)} = I$ if $p > q$, and (b) holds by applying \eqref{eq:matrix-consensus2}. The above equality leads to
\begin{align*}
\|\vphi^\ko\|^2_F&
\stackrel{(a)}{\leq}\gamma^2(k-m\tau)
\sum_{l=1}^{k-m\tau}\left\|\left(\prod_{j=k-l+1}^{k-1}\vW^{(j)}\right)\vW^{(k-l)}_g\vpsi^{(k-l)}\right\|^2_F\\
&\stackrel{(b)}{\leq} 2\tau\gamma^2
\sum_{l=1}^{k-m\tau}\left\|\left(\prod_{j=k-l+1}^{k-1}\vW^{(j)}\right)\vW^{(k-l)}_g\vpsi^{(k-l)}\right\|^2_F\\
&\stackrel{(c)}{\leq}
4\tau\gamma^2
\sum_{l=1}^{k-m\tau}\left\|\vpsi^{(k-l)}\right\|^2_F.
\end{align*}
Here, inequality (a) holds due to Jensen's inequality,  (b) holds since $k-m\tau \le 2 \tau$, and (c) holds due to the fact that $\left(\prod_{j=k-l+1}^{k-1}\vW^{(j)}\right)\vW^{(k-l)}_g \in \mathscr{F}$ and Lemma \ref{lem:F-matrix-upper-bound}.  Summing up the above inequality over $k=\tau,\tau+1,\ldots,T$ and dividing it by $T+1$, we get
\begin{equation}
\label{eq:phi_bound_ineq}
\frac{1}{T+1}\sum_{k=\tau}^T \|\vphi^\ko\|^2_F\leq
4\tau\gamma^2
\frac{1}{T+1}\sum_{k=\tau}^T\sum_{l=1}^{k-m\tau}\left\|\vpsi^{(k-l)}\right\|^2_F
\leq 8\tau^2\gamma^2
\avekT\left\|\vpsi^{(k)}\right\|^2_F.
\end{equation}
{Adding $\frac{1}{T+1}\sum_{k=0}^{\tau-1}\|\vphi^\ko\|_F^2$ to the both sides of \eqref{eq:phi_bound_ineq}, } then taking expectations to both sides of the above inequality, we have
\begin{equation}
\label{eq:ave-sq-phi}
\avekT \bbE\|\vphi^\ko\|^2_F
\leq \frac{1}{T+1}\sum_{k=0}^\tau\bbE\|\vphi^\ko\|^2_F+8\tau^2\gamma^2\avekT\bbE\left\|\vpsi^{(k)}\right\|^2_F
\end{equation}

\textbf{Step II.} In this step, we derive the bound on the term $\avekT\|\vpsi^\ko\|^2_F$. Recall that $\|\vpsi^\ko\|^2_F=\|\vg^\ko-\bbone_n (\bar\vg^\ko)^\top\|^2_F+\|\vh^\ko-\bbone_n (\bar\vh^\ko)^\top\|^2_F$. With $\mathcal{F}^{(k-1)}$ defined in \eqref{eq:hist-sigma-alg}, we have
\begin{align*}
&\quad\bbE[\|\vg^\ko-\bbone_n (\bar\vg^\ko)^\top\|^2_F|\mathcal{F}^{(k-1)}]\\
&=\bbE\left[\left.\left\|(\vI_n-\frac{1}{n}\bbone_n\bbone_n^\top)\vg^\ko\right\|^2_F\right|\mathcal{F}^{(k-1)}\right]\\
&=\bbE\left[\left.\left\|(\vI_n-\frac{1}{n}\bbone_n\bbone_n^\top)(\vg^\ko-\nabla F(\vx^\ko)+\nabla F(\vx^\ko))\right\|^2_F\right|\mathcal{F}^{(k-1)}\right]\\
&\stackrel{(a)}{=}\bbE\left[
\left.\|(\vI_n-\frac{1}{n}\bbone_n\bbone_n^\top)(\vg^\ko-\nabla F(\vx^\ko))\|^2_F\right|\mathcal{F}^{(k-1)}
\right]
+
\|(\vI_n-\frac{1}{n}\bbone_n\bbone_n^\top)\nabla F(\vx^\ko)\|^2_F\\
&\leq \bbE\left[
\left.\|\vg^\ko-\nabla F(\vx^\ko)\|^2_F\right|\mathcal{F}^{(k-1)}
\right]
+\|(\vI_n-\frac{1}{n}\bbone_n\bbone_n^\top)(\nabla F(\vx^\ko)-\nabla F(\bar\vx^\ko)+\nabla F(\bar\vx^\ko))\|^2_F\\
&\stackrel{(b)}{\leq}
n \sigma^2 + 2\|(\vI_n-\frac{1}{n}\bbone_n\bbone_n^\top)(\nabla F(\vx^\ko)-\nabla F(\bar\vx^\ko))\|^2_F+2\|(\vI_n-\frac{1}{n}\bbone_n\bbone_n^\top)(\nabla F(\bar\vx^\ko))\|^2_F\\
&\leq
n \sigma^2+2\|\nabla F(\vx^\ko)-\nabla F(\bar\vx^\ko)\|^2_F+2\|(\vI_n-\frac{1}{n}\bbone_n\bbone_n^\top)(\nabla F(\bar\vx^\ko))\|^2_F\\
&\stackrel{(c)}{\leq}
n\sigma^2+2L^2\|\vx^\ko-\bbone_n (\bar\vx^\ko)^\top\|^2_F+2nb^2.
\end{align*}
In the above inequality, (a) holds because of Assumption \ref{ass:grad-noise}, (b) holds because of Assumption \ref{ass:grad-noise} and Jensen's inequality, and (c) holds due to Assumption \ref{ass:Lip} and Assumption \ref{ass:data-hetero}. By taking expectation over $\mathcal{F}^{(k-1)}$, we have
\begin{align*}
\bbE\|\vg^\ko-\bbone_n (\bar\vg^\ko)^\top\|^2_F \le 2L^2\bbE\|\vx^\ko-\bbone_n (\bar\vx^\ko)^\top\|^2_F + n\sigma^2 + 2nb^2.
\end{align*}
A similar bound can also be derived for $\bbE\|\vh^\ko-\bbone_n (\bar\vh^\ko)^\top\|^2_F$. As a result, we achieve 
$$\bbE\|\vpsi^\ko\|^2_F\leq 2n\sigma^2+4nb^2+2L^2\bbE\|\vphi^\ko\|^2_F.$$
Summing up the inequality for $k=0,1,2,\ldots,T$ and then dividing the result by $T+1$, we get 
\begin{equation}
\label{eq:ave-psi-bound}
\avekT\bbE\|\vpsi^\ko\|^2_F\leq 2n\sigma^2+4nb^2+2L^2\avekT\bbE\|\vphi^\ko\|^2_F.
\end{equation}
\textbf{Step III.} In this step, we will derive \eqref{eq-consensus} based on \eqref{eq:ave-sq-phi} and \eqref{eq:ave-psi-bound}. Substituting \eqref{eq:ave-psi-bound} to \eqref{eq:ave-sq-phi}, we get
\begin{align}
\avekT \bbE\|\vphi^\ko\|^2_F
\leq \frac{1}{T+1}\sum_{k=0}^\tau\bbE\|\vphi^\ko\|^2_F+\frac{16\tau^2\gamma^2L^2}{T+1}\sum_{k=0}^T \bbE\|\vphi^\ko\|^2_F + 16n\tau^2\gamma^2(\sigma^2 + 2b^2) 
\end{align}
When $\gamma \le \frac{1}{8\tau L}$, it holds that $16\tau^2L^2\gamma^2 \le \frac{1}{2}$. Regrouping terms associated with $\avekT \bbE\|\vphi^\ko\|^2_F$, we get \eqref{eq-consensus}.
\end{proof}

With Lemmas \ref{lem:descent} and \ref{lem:consenus-dsgd}, we can establish the following convergence property of DSGD-CECA. 
\begin{theorem}[\sc Convergence property]
Suppose Assumptions \ref{ass:Lip}-\ref{ass:data-hetero} hold and $\bvx_i^\ko = \bar{\vx}^\ko$, $\bvy_i^\ko = \bar{\vy}^\ko$ for $k < \tau$ by conducting global averaging in the first $\tau$ iterations. If learning rate $\gamma$ satisfies 
\begin{align}
\gamma = \frac{1}{\left( \frac{2n\Delta}{L\sigma^2(T+1)}\right)^{-\frac{1}{2}} + \left( \frac{\Delta}{24L^2\tau^2(\sigma^2+2b^2)(T+1)}\right)^{-\frac{1}{3}} + 8\tau L}
\end{align}
where $\Delta = \bbE f(\bar\vx^{0}) - f^\star$, then DSGD-CECA will converge as follows
\begin{align}\label{znw9237-0}
\frac{1}{T+1}\sum_{k=0}^T \bbE\|\nabla f(\bar\vx^\ko)\|^2 
&\le 16 \left( \frac{\Delta L \sigma^2}{n(T+1)} \right)^{\frac{1}{2}} + 24 \left(\frac{\Delta^2 L^2 \tau^2(\sigma^2 + 2b^2)}{(T+1)^2}\right)^{\frac{1}{3}} + \frac{32\tau \Delta L}{T+1}
\end{align}
\end{theorem}
\begin{proof}
Recall that $\|\vphi^\ko\|^2_F=\|\vx^\ko-\bbone_n (\bar\vx^\ko)^\top\|^2_F+\|\vy^\ko-\bbone_n (\bar\vy^\ko)^\top\|^2_F$. By averaging  \eqref{eq-descent-lamma} over $k=0,1,\cdots, T$, we have 
\begin{align}\label{znw9237}
\frac{1}{T+1}\sum_{k=0}^T \bbE\|\nabla f(\bar\vz^\ko)\|^2 &\le \frac{4(\bbE f(\bar\vx^{0}) - f^\star)}{\gamma (T+1)} + \frac{3L^2}{n(T+1)}\sum_{k=0}^T\bbE\|\vphi^\ko\|^2_F + \frac{2\gamma L \sigma^2}{n} \nonumber \\
&\le \frac{4(\bbE f(\bar\vx^{0}) - f^\star)}{\gamma (T+1)} + 96 L^2\tau^2\gamma^2(\sigma^2 + 2b^2)  + \frac{2\gamma L \sigma^2}{n},
\end{align}
where the last inequality holds because of inequality \eqref{eq-consensus} and the fact that $\bvx_i^\ko = \bar{\vx}^\ko$, $\bvy_i^\ko = \bar{\vy}^\ko$ for $k < \tau$ due to the global averaging in the first $\tau$ iterations. We next let $\Delta := \bbE f(\bar\vx^{0}) - f^\star$ and define
\begin{align}
    \gamma_1 = \left( \frac{2n\Delta}{L\sigma^2(T+1)}\right)^{\frac{1}{2}}, \quad \gamma_2 = \left( \frac{\Delta}{24L^2\tau^2(\sigma^2+2b^2)(T+1)}\right)^{\frac{1}{3}}.
\end{align}
If we set 
$$\gamma = \frac{1}{\gamma_1^{-1} + \gamma_2^{-1} + 8\tau L},$$
it holds that $\gamma \le \min\{\gamma_1, \gamma_2, \frac{1}{8\tau L}, \frac{1}{4L}\}$. Substituting the above $\gamma$ to \eqref{znw9237}, we achieve
\begin{align}\label{znw9237-2}
\frac{1}{T+1}\sum_{k=0}^T \bbE\|\nabla f(\bar\vz^\ko)\|^2 &\le \frac{4\Delta}{T+1}(\gamma_1^{-1} + \gamma_2^{-1} + 8\tau L) + 96 L^2\tau^2\gamma_2^2(\sigma^2 + 2b^2)  + \frac{2\gamma_1 L \sigma^2}{n} \nonumber \\
&= 16 \left( \frac{\Delta L \sigma^2}{n(T+1)} \right)^{\frac{1}{2}} + 24 \left(\frac{\Delta^2 L^2 \tau^2(\sigma^2 + 2b^2)}{(T+1)^2}\right)^{\frac{1}{3}} + \frac{32\tau \Delta L}{T+1}.
\end{align}
Substituting $\bar\vx^\ko = \bar\vz^\ko$ to the above inequality (see Lemma \ref{lem:bar-eq-iter}), we achieve \eqref{znw9237-0}. 
\end{proof}

\section{Additional experiments}\label{appendix: exp}
\textbf{CIFAR-10:} We utilize the ResNet-18 model \cite{he2016deep} implemented by \cite{kuangliu}. Similar to the MNIST experiments, we employ BlueFog for decentralized training using 5 NVIDIA GeForce RTX 2080 GPUs. The training process consists of 130 epochs without momentum, with a weight decay of $10^{-4}$. A local batch size of 64 is used, and the base learning rate is set to 0.01. The learning rate is reduced by a factor of 10 at the 50th, 100th, and 120th epochs. Data augmentation is performed similarly to the method described in the work \cite{kuangliu}. Please refer to Fig.~\ref{fig:cifar10} for a comparison of the test accuracy between O.-P. Exp and DSGD-CECA-2P. It is noteworthy that DSGD-CECA-2P outperforms O.-P. Exp in terms of test accuracy.

\setlength{\tabcolsep}{4pt}
\begin{table}[h!]
    \centering 
    \caption{\small Comparison of test accuracy(\%) with O.-P. Exp and DSGD-CECA-2P over MNIST and CIFAR-10 datasets.}
    \vspace{2mm}
	\begin{tabular}{rccccc}
		\toprule
		&\textbf{Topology} &  \textbf{MNIST Test Acc.} &\textbf{CIFAR-10 Test Acc.}\\ \midrule
		&O.-P. Exp.   &  98.33   & 90.99  \\  
            &DSGD-CECA-2P   & \textbf{98.50} &\textbf{92.07} \\
		\bottomrule
	\end{tabular}
	\label{table:deep_learning}
\end{table}

\begin{figure}[ht]
\vskip -0in
\begin{center}
\includegraphics[width=0.4\columnwidth]{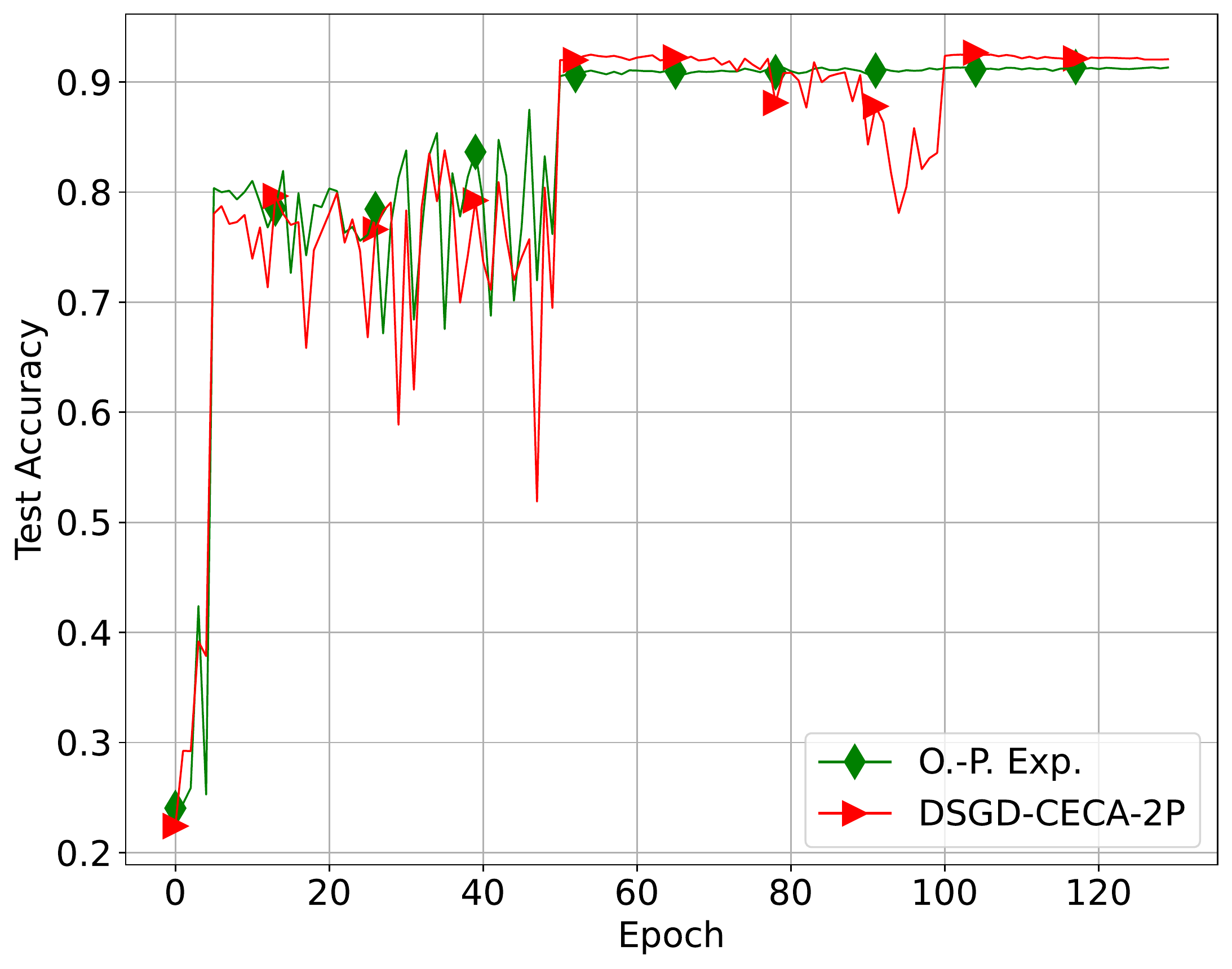}
\caption{Test accuracy of O.-P. Exp and DSGD-CECA-2P algorithms for ResNet-18 on CIFAR-10.}
\label{fig:cifar10}
\end{center}
\vskip -0.45in
\end{figure}

\end{document}